\documentclass[a4paper,fleqn]{cas-sc}

\usepackage{amsmath,amsfonts}
\usepackage{algorithm}
\usepackage{algorithmic}
\usepackage{array}
\usepackage{color}
\usepackage[caption=false,font=normalsize,labelfont=sf,textfont=sf]{subfig}
\usepackage{textcomp}
\usepackage{stfloats}
\usepackage{url}
\usepackage{verbatim}
\usepackage{graphicx}
\usepackage{multirow}
\usepackage{hyperref}
\usepackage{bm}
\usepackage{amsthm}
\usepackage{amssymb}
\theoremstyle{definition}

\newtheorem{theorem}{Theorem}
\newtheorem{lemma}{Lemma}

\newtheorem{proposition}[theorem]{Proposition}
\usepackage[T1]{fontenc}
\usepackage{orcidlink}
\hypersetup{hidelinks,
	colorlinks=true,
	allcolors=black,
	pdfstartview=Fit,
	breaklinks=true}
\def\BibTeX{{\rm B\kern-.05em{\sc i\kern-.025em b}\kern-.08em
    T\kern-.1667em\lower.7ex\hbox{E}\kern-.125emX}}
\usepackage{balance}

\usepackage{nomencl}
\makenomenclature

\usepackage{bbding}

\usepackage{graphicx}

\usepackage{listings}
\usepackage{courier}
\usepackage{todonotes}
\usepackage{multicol}
\usepackage{multirow}

\usepackage{rotating}

\definecolor{codegreen}{rgb}{0,0.6,0}
\definecolor{codegray}{rgb}{0.5,0.5,0.5}
\definecolor{codepurple}{rgb}{0.58,0,0.82}
\definecolor{backcolour}{rgb}{0.95,0.95,0.92}

\lstdefinestyle{mystyle}{
	backgroundcolor=\color{backcolour},   
	commentstyle=\color{codegreen},
	keywordstyle=\color{magenta},
	numberstyle=\tiny\color{codegray},
	stringstyle=\color{codepurple},
	basicstyle=\small\ttfamily,
	breakatwhitespace=false,         
	breaklines=true,                 
	captionpos=b,                    
	keepspaces=true,                 
	numbers=left,                    
	numbersep=5pt,                  
	showspaces=false,                
	showstringspaces=false,
	showtabs=false,                  
	tabsize=2
}
\usepackage{graphicx}
\usepackage{caption}

\usepackage{booktabs}

\begin{document}
\let\WriteBookmarks\relax
\def\floatpagepagefraction{1}
\def\textpagefraction{.001}
\shorttitle{Efficient State Space Model}
\shortauthors{T. Liang et~al.}

\title [mode = title]{Efficient State Space Model via Fast Tensor Convolution and Block Diagonalization}   

\author{Tongyi Liang}[style=chinese,orcid=0000-0001-8617-2396]
\ead{tyliang4-c@my.cityu.edu.hk}

\credit{Conceptualization, Methodology, Software, Writing}

\affiliation{organization={Department of Systems Engineering, City University of Hong Kong},
                city={Hong Kong,  SAR, },
                country={China}}

\author{Han-Xiong Li}[style=chinese,orcid=0000-0002-0707-5940]
\cormark[1]
\ead{mehxli@cityu.edu.hk}
\credit{Supervision, Funding acquisition}
	
\cortext[cor1]{Corresponding author}

\begin{abstract}
	Existing models encounter bottlenecks in balancing performance and computational efficiency when modeling long sequences. Although the state space model (SSM) has achieved remarkable success in handling long sequence tasks, it still faces the problem of large number of parameters. In order to further improve the efficiency of SSM, we propose a new state space layer based on multiple-input multiple-output SSM, called efficient SSM (eSSM).
    Our eSSM is built on the convolutional representation of multi-input and multi-input  (MIMO) SSM. We propose a variety of effective strategies to improve the computational efficiency. The diagonalization of the system matrix first decouples the original system. Then a fast tensor convolution is proposed based on the fast Fourier transform. In addition, the block diagonalization of the SSM further reduces the model parameters and improves the model flexibility.
    Extensive experimental results show that the performance of the proposed model on multiple databases matches the performance of state-of-the-art models, such as S4, and is significantly better than Transformers and LSTM.
    In the model efficiency benchmark, the parameters of eSSM are only 12.89\% of LSTM and 13.24\% of Mamba. The training speed of eSSM is 3.94 times faster than LSTM and 1.35 times faster than Mamba.
    Code is available at: \href{https://github.com/leonty1/essm}{https://github.com/leonty1/essm}.
\end{abstract}

\begin{graphicalabstract}
\includegraphics[width=1\textwidth]{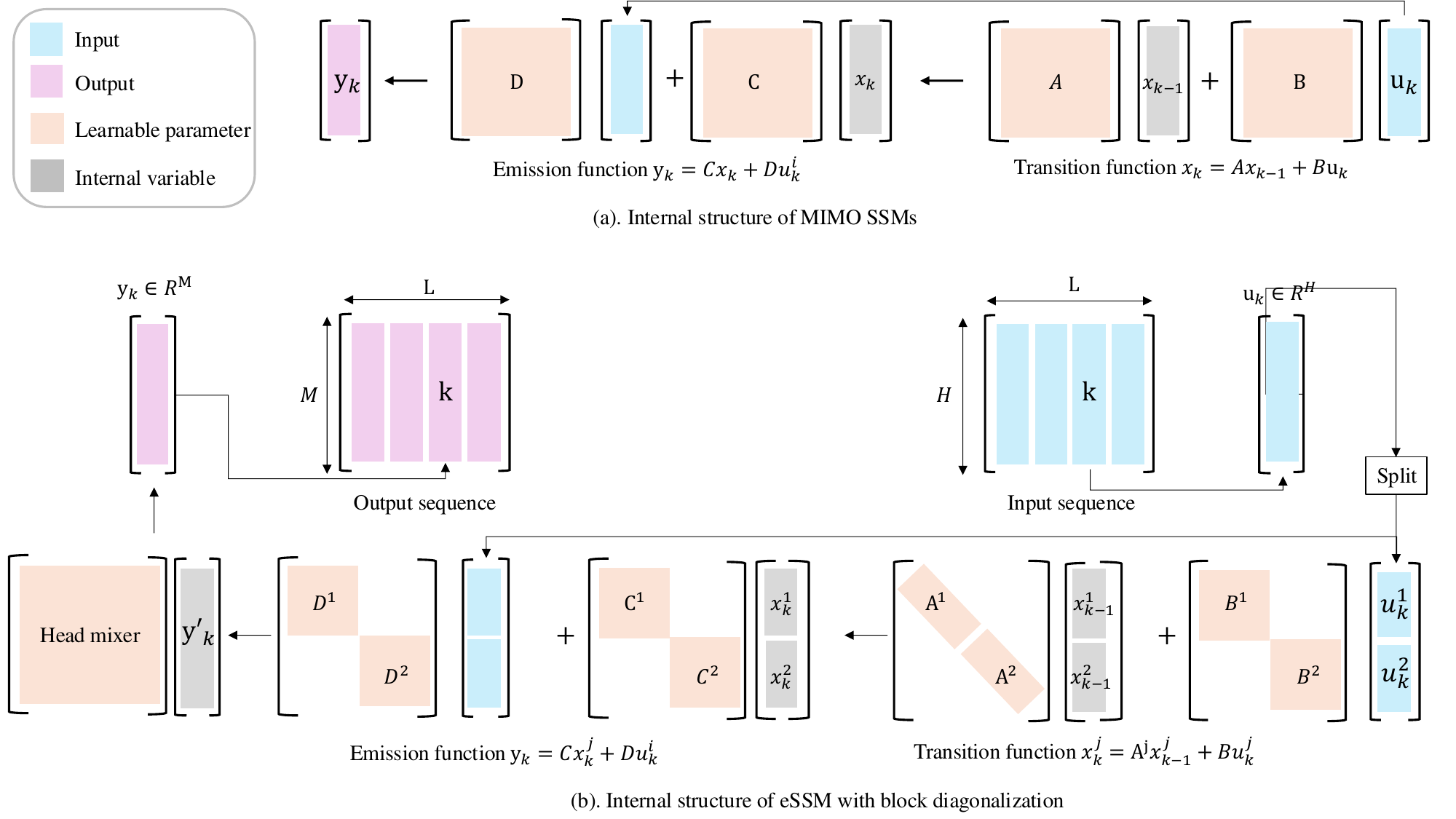}
\end{graphicalabstract}

\begin{highlights}
\item Proposed a novel MIMO state space model (SSM) architecture: We introduce a new state space layer called eSSM, utilizing the convolutional representation of multi-input multi-output (MIMO) SSMs, enabling efficient long sequence modeling with fewer parameters and faster training speed.
\item Improved computational efficiency with advanced techniques: The eSSM employs diagonalization and fast tensor convolution to reduce parameter size, along with bidirectional kernels and block diagonalization strategies that enhance both computational efficiency and model diversity without adding complexity.
\item Comprehensive evaluation against SOTA models: Our experiments show that eSSM achieves accuracy comparable to state-of-the-art models while significantly reducing parameter size and training time, making it a highly efficient solution for long sequence modeling tasks.
\end{highlights}


\maketitle

\section{Introduction}
Sequence modeling with long-range dependency is a crucial and challenging problem in many application fields of machine learning, e.g., natural language processing (NLP) \cite{chowdhary2020natural}, computer vision (CV) \cite{bautista2022scene}, and time series forecasting \cite{petropoulos2022forecasting}. Extensive research has been conducted to mitigate this problem. Among various works, deep learning has shown its compelling ability in sequence modeling.

Recurrent Neural Networks (RNNs) employ a recursive architecture with a memory mechanism that compresses history information. Compared with standard feedforward neural networks, RNNs can make flexible inferences at any time length. However, the gradient vanishing and exploding problem make them challenging to train \cite{bengio2013advances}. Subsequent works, like LSTM \cite{hochreiter1997long} and bidirectional RNN \cite{schuster1997bidirectional}, have attempted to address these issues. However, the sequential nature of RNNs makes them unable to be trained efficiently in a parallel setting because of the backpropagation across the time chain.

Transformer with attention mechanism avoids the problems in RNNs and realizes parallel computing \cite{vaswani2017attention}. Thereafter, it achieves remarkable success in extend fields, including demand forecasting \cite{zhou2023graph}, traffic flow forecasting \cite{du2024multi}. However, the quadratic complexity of length prevents the Transformer from exhibiting efficient performance in modeling long sequences. Many efficient variants have been introduced with reduced complexity to tackle this issue, such as Informer \cite{zhou2021informer}, and Crossformer \cite{zhang2022crossformer}. Unfortunately, these efficient variants provide unsatisfying results on long-sequence modeling tasks, worse than the original Transformer \cite{tay2020long}.

\begin{figure*}[htbp]
\centering
\subfloat[]{\includegraphics[width=0.32\linewidth]{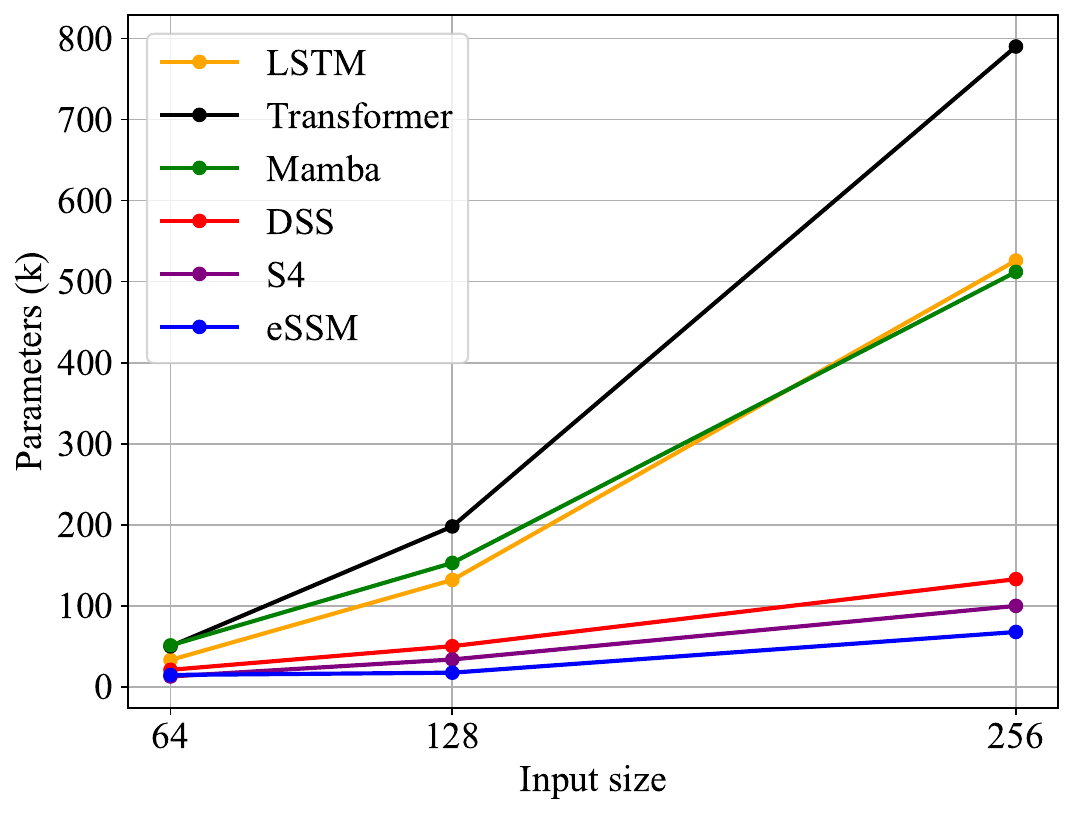}%
\label{fig_first_case}}
\hfil
\subfloat[]{\includegraphics[width=0.32\linewidth]{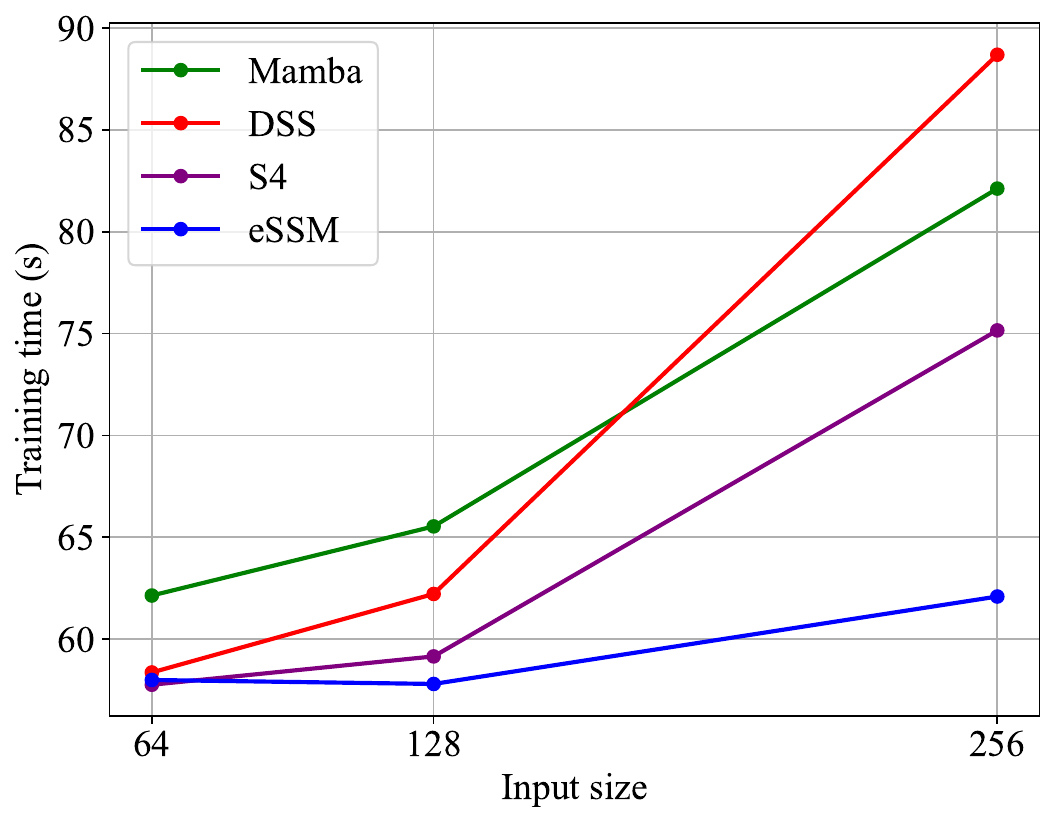}%
\label{fig_second_case}}
\hfil
\subfloat[]{\includegraphics[width=0.32\linewidth]{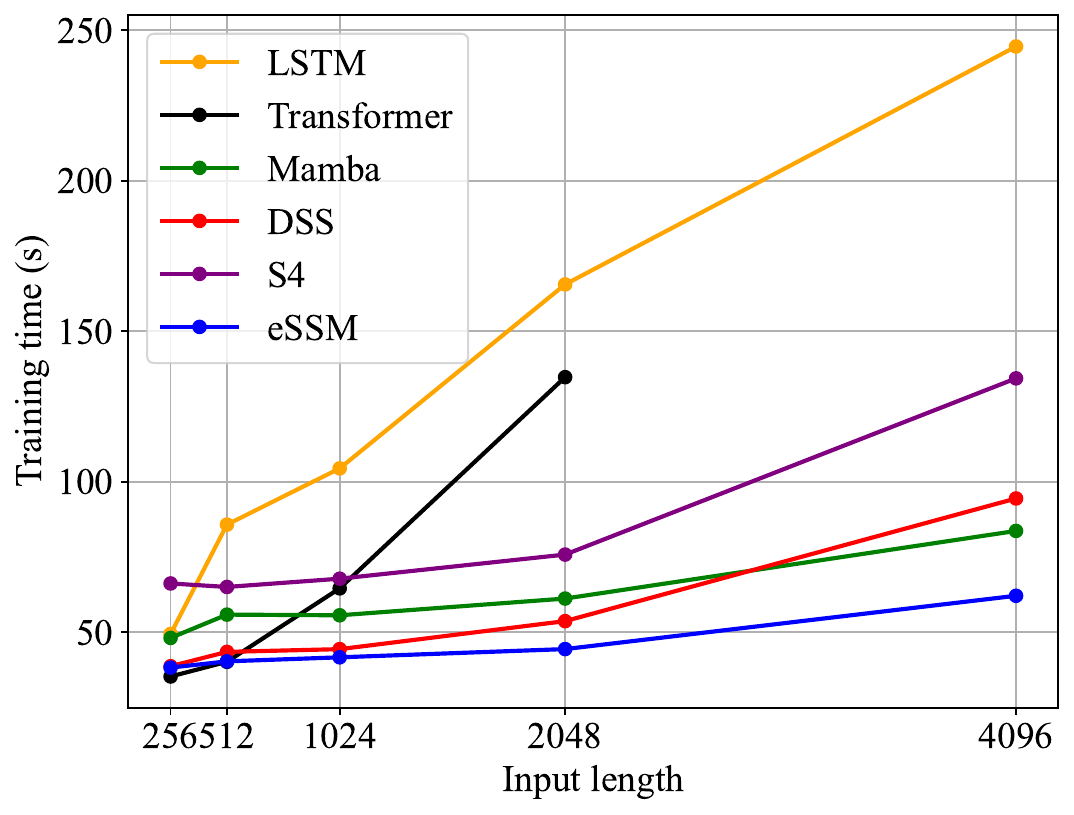}%
\label{fig_third_case}}
\caption{Evaluating the efficiency of eSSM and baseline methods for sequence modeling: (a). Model parameters vs. different input sizes. (b). Training time per epoch  vs. different input sizes. (c). Training time per epoch vs. different input sequence lengths. Our eSSM is the most efficient model, outperforming Transformer, LSTM, and other SSM-based models.}
\label{essm-fig: efficient comparison}
\end{figure*}

\begin{table}[]
	\caption{Comparison of different SSM-based models. S4 and S4D are based on state space models (SSMs) with single input single output (SISO); S5 and our eSSM are based on SSMs with multiple input multiple output (MIMO); S5 cannot work in convolution mode, but our eSSM makes it.}
	\label{tab:compare-ssm}
	\centering
	\begin{tabular}{cccc}
		\hline
		Model & Type & Recurrent  & Convolutional 	\\ \hline
		S4    & SISO & \Checkmark & \Checkmark   	\\
		S4D   & SISO & \Checkmark & \Checkmark  	\\ \hline
		S5    & MIMO & \Checkmark & \XSolid    		\\
		Our eSSM  & MIMO & \Checkmark & \Checkmark  	\\ \hline
	\end{tabular}
\end{table}

Recently, state space models (SSMs) based neural networks have made significant progress in long-sequence modeling. LSSL \cite{gu2021combining} incorporates the HiPPO theory \cite{gu2020hippo} and outperforms the Transformer family models in many sequence tasks. Hereafter, some variants of SSMs are proposed, like S4 \cite{gu2021efficiently}, DSS \cite{gupta2022diagonal}, S4D \cite{gu2022parameterization} and so on. They all show remarkable abilities in modeling long sequences on various tasks, including time series forecasting, speech recognition, and image classification.

However, there still exist some research gaps to be addressed.
\begin{itemize}
	\item Architectural limitations. Existing models, like S4, are based on single-input and single-output (SISO) SSMs stacked parallel for modeling multi-variate sequences. An additional linear layer is needed to mix features from different SISO SSMs. This approach not only increases the complexity of the model, but also may lead to insufficient information fusion, affecting the model's effective modeling of multivariate sequences.
	\item Theoretical incompleteness. SSMs with multi-input and multi-output (MIMO) theoretically do not require additional linear layers. However,  to our best knowledge, S5 \cite{smith2022simplified} is currently the only work based on MIMO SSMs. As shown in Table \ref{tab:compare-ssm}, S5 uses the recurrent representation for training. However, the convolution representation of MIMO SSMs has not been reported yet. The lack of such research not only limits our understanding of MIMO SSM, but also limits its potential in practical wide applications.
	\item Challenges in computational efficiency. Existing SSM-based models, including S4 and S5, have large parameter sizes, resulting in high consumption of computing resources and extended training time. Therefore, further reducing the number of model parameters to improve computational efficiency and storage efficiency is an important issue that needs to be addressed in current research.
\end{itemize}

In this work, we introduce an \underline{e}fficient \underline{s}tate \underline{s}pace \underline{m}odel (eSSM) designed to solve the abovementioned issues. We model the dynamics of sequential data using continuous MIMO SSMs, which have the properties of fewer parameters than discrete neural networks. eSSM is built on MIMO SSMs, which, unlike SISO SSM, can directly model the multi-variate sequences. Then, a discrete representation can be obtained through discretization. The discretized SSMs are recurrent models that can make flexible inferences as RNNs. Furthermore, to train the proposed model efficiently, we derive its convolutional representations, which enables the training process to be more efficient with parallel computing.

We reduce the number of model parameters and improve the computational efficiency of the model by diagonalization and fast tensor convolution via fast Fourier transform (FFT) strategies. Specifically, during the training phase, the diagonal SSMs in convolutional representations significantly reduce time and space complexity. Afterwards, we disentangle and transform the original convolution operation between the state kernel and multi-variate input sequence equivalently. Then, FFT is applied for the disentangled convolution calculation.

The linear SSMs used in this work are causal systems, which may limit their applications to non-causal cases, such as image classification. To fix this, we propose the non-causal variant with a bidirectional kernel. Besides, motivated by Multi-Head Attention \cite{vaswani2017attention}, we design the block-diagonal eSSM. The original input vector is split into multiple sub-vectors and then modeled by eSSM. Outputs of all eSSM are concatenated and mixed by a linear projection function.

We evaluate the effectiveness and performance of eSSM using extensive datasets, including classification tasks over text, image, and audio. Our models outperform Transformer variants and match the performance of the state-of-the-art models. In terms of computational efficiency, eSSM has far fewer parameters than various baselines, including Transformer, LSTM, and Mamba. And eSSM also has significantly faster training speed than them, as shown in \autoref{essm-fig: efficient comparison}.

The main contributions of our work are summarized as follows.
\begin{enumerate}
	\item We extended the existing SSM family and designed a new state space layer called eSSM. Using the convolutional representation of MIMO SSM, we achieved efficient long sequence modeling, which has the characteristics of few parameters and fast training speed (\autoref{essm sec: Preliminary}).
	\item We have proposed a variety of strategies to effectively improve the efficiency of the model. Using diagonalization and fast tensor convolution, we can significantly reduce the number of parameters and improve the training and inference speed. The bidirectional kernel without introducing additional parameters and the block diagonalization strategy further improve the efficiency of eSSM and increase its diversity (\autoref{essm sec: Methodology}).
	\item We conduct a comprehensive comparison and analysis of the proposed model with various SOTA baselines in terms of accuracy and computational efficiency. The results show that our eSSM can achieve accuracy matching that of SOTA methods, but our model has significantly fewer parameters and faster training speed. (\autoref{essm sec: Experiments}).
\end{enumerate}
\section{Related Work}
\textbf{CNN for long-sequence modeling}. Convolutional Neural Networks (CNNs) achieve great success in CV tasks and significantly impact sequence modeling. For example, Li et al. \cite{li2018convolutional} construct a hierarchical CNN model for human motion prediction. However, CNNs require large sizes of parameters to learn long-range dependency in a long-sequence task. Romero \cite{romero2021ckconv} introduces a continuous kernel convolution (CKConv) that generates a long kernel for an arbitrarily long sequence. Motivated by S4 \cite{gu2021efficiently}, Li et al.\cite{li2022makes}
try to learn the global convolutional kernel directly with a decaying structure, which can be seen as an easier alternative to SSMs. Fu et al. \cite{fu2023simple} find that simple regularizations, squashing, and smoothing can have long convolutional kernels with high accuracy.

\textbf{Continuous-time RNNs}. RNNs are the primary choice for modeling sequences \cite{gu2024modeling}. Extensive work has been done to improve RNNs by designing gating mechanisms \cite{hochreiter1997long} and memory representations \cite{voelker2019legendre}. Neural ordinary differential equations (Neural-ODE) have recently attracted researchers' attention \cite{chen2018neural}. This special kind of continuous RNNs, also including ODE-RNN \cite{rubanova2019latent} and Neural-CDE \cite{kidger2020neural}, perform well in sequential tasks.

\textbf{Efficient Transformer}. The machine learning community has made great efforts to improve the computation efficiency of the vanilla Transformer \cite{vaswani2017attention}. Introducing sparsity into the Attention mechanism is a way to reduce complexity. Examples include Longformer \cite{beltagy2020longformer}, BigBird \cite{zaheer2020big}, and Sparse Transformer \cite{child2019generating}. Besides, some works like Performer \cite{choromanski2020rethinking} and Linear Transformer \cite{katharopoulos2020transformers} aim to approximate or replace the original attention matrix with linear computation. Some approaches reduce the complexity by computing attention distributions for clustered queries \cite{vyas2020fast} or replace the attention with low-rank approximation \cite{guo2019low}.

\textbf{SSMs-based NN}. SSMs-based deep models are recent breakthroughs in long-sequence modeling. In their inspiring work, Gu et al. \cite{gu2021combining} propose LSSL, which combines recurrence, convolution, and time continuity. Later, to improve the computational efficiency of LSSL, Gu et al. \cite{gu2021efficiently} designed an efficient layer called S4 by transforming the system matrix into a structured normal plus low-rank matrix. Gupta et al. \cite{gupta2022diagonal} propose a DSS that matches the performance of S4 by assuming the system matrix is diagonal. Afterward, several modified variants of S4 and DSS are introduced as S4D \cite{gu2022parameterization}, GSS \cite{mehta2022long}, Mamba \cite{gu2023mamba}, and so on.

The main difference between our work and the above SSMs-based models is that they are based on SISO SSMs, but our work is built directly on MIMO SSMs. It should be noted that SISO SSMs are a particular case of MIMO SSMs. The most similar work is S5 \cite{smith2022simplified}, which uses MIMO SSMs. In S5, the authors use parallel scans to learn under the recurrent representation of MIMO SSM. In our work, however, eSSM learns under the convolutional representation of MIMO SSM with proposed efficient strategies. Appendix C introduces the comparison of related models.
\section{Preliminary}\label{essm sec: Preliminary}
\subsection{Problem Description}
The sequence modeling aims to construct a map $\mathcal{F}$  with parameter $\theta$ between an input sequence $\mathcal{U}$ and output $\mathcal{Y}$. Therefore, this problem is defined as:
\begin{equation}
	\label{problem_sq}
	\mathcal{F}_\theta: \mathcal{U}\mapsto \mathcal{Y}
\end{equation}

\subsection{Continuous SSMs}
Continuous linear SSMs are classical models that describe the dynamics of linear systems. Typically, a continuous linear dynamic system with multi-input and multi-output (MIMO) is defined as the following general form:
\begin{align}
	\dot{x}(t)=Ax(t) + Bu(t), \quad
	y(t) = Cx(t) + Du(t).
	\label{continuous_ssm}
\end{align}
where system matrix $A \in \mathbb{R}^{N\times N}$, input matrix $B \in \mathbb{R}^{N\times H}$, output matrix $C \in \mathbb{R}^{M\times N}$, direct transition matrix $D \in \mathbb{R}^{M\times H}$. $u(t) \in \mathbb{R}^H, x(t)\in \mathbb{R}^N$, and $y(t)\in \mathbb{R}^M$ are input, state, and output of the system. The internal computation structure is shown in Fig. \ref{InternalStructureeSSM} (a).

\begin{figure*}
	\centering
	\includegraphics[width=1\textwidth]{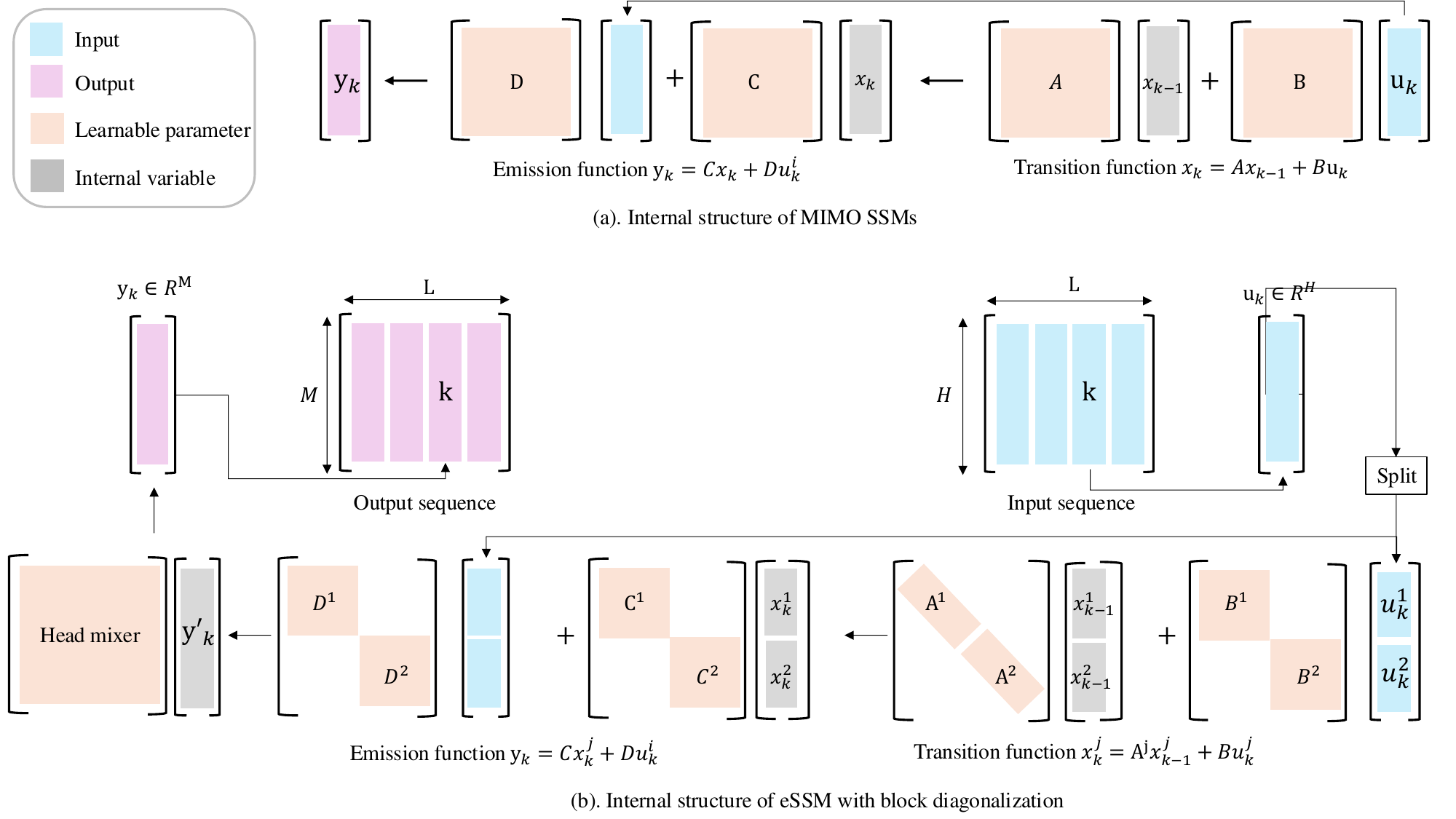}
	\caption{The internal structure of state space models (SSMs) with multi-input and multi-output (MIMO) (a) and Multi-Head eSSM (b).}
	\label{InternalStructureeSSM}
\end{figure*}

In control theory, state $x$ is a variable that represents the dynamic state of the system. For example, when SSMs are used to model the moving process of an object, the state $x$ could include the velocity and acceleration of this object. However, in pure data cases, assigning the physical meanings for the state is intractable. We can understand this model from the perspective of polynomial projection. That is, given the sequence $\mathcal{U}$ generated from an unknown function $f$, $x$ are the coefficients of projected polynomials. Thus, $x$ can be seen as the compression of all history \cite{gu2020hippo}. In other words, $x$ could be interpreted as the memory of history.

It should be noted that SSMs with single-input and single-output (SISO) are the particular case of MIMO SSMs. Typically, the SISO SSMs, e.g., S4 \cite{gu2021efficiently} and DSS \cite{gupta2022diagonal}, are stacked in parallel and require an additional linear layer to mix features. However, MIMO SSMs model the multi-dimensional sequences directly, and no more feature-mixing layers are needed in theory.

\subsection{Discrete SSMs}
We discrete the continuous SSMs \eqref{continuous_ssm} to learn from the discrete sequence. The discrete SSMs with a time step $\Delta\in \mathbb{R_{+}}$ are expressed as follows.
\begin{align}
	x_{k}=\bar{A} x_{k-1} + \bar{B}u_k, \quad
	y_k = \bar{C}x_k + \bar{D}u_k.
	\label{discret_ssm}
\end{align}
There are many standard discretization technologies available. More details about discretization methods, including the Generalized Bilinear Transformation method (GBT), are introduced in Appendix B.2. Coefficients matrix $\bar{C}=C$, and $\bar{D}=D$, but $\bar{A}$ and $\bar{B}$ would differ in different methods.  In this work, we use the Zero-order Hold (ZOH) method \cite{gu2012discrete}, where
\begin{align}
	\bar{A}= e^{A\Delta}, \quad
	\bar{B}= A^{-1}(e^{A\Delta}-I)B.
	\label{discret_zoh}
\end{align}

Discrete SSMs are recurrent models, and state $x$ can be an analogy to the hidden state in RNNs. The recurrent models have excellent advantages in flexibility during the inference process but are hard to train.

\subsection{Convolutional SSMs}
The convolution nature of SSMs makes it possible to train SSMs efficiently. Let the initial state be $x_{0}$, and we then unroll the recursive process Eq. \eqref{discret_ssm} along time.
\begin{align*}
	&x_{1}=\bar{A}x_{0} + \bar{B}u_{1} \\
	&\vdots \\
	&x_{k}=\underbrace{\bar{A}^{k}x_{0}}_{ bias}+\underbrace{\bar{A}^{k-1}\bar{B}u_{1}+\cdots +\bar{B}u_{k}}_{convolution\, operation}\\
	&\vdots \\
	&x_{L}=\bar{A}^{L}x_{0}+\bar{A}^{L-1}\bar{B}u_{1}+\cdots +\bar{B}u_{L}
\end{align*}

State $x_k$ is the summation of a bias and a convolution operation of all history inputs. We define the state kernel as Eq. \eqref{state_kernel} for simplicity.

\begin{equation}\label{state_kernel}
	K \triangleq \left(\bar{B}, \bar{A}\bar{B},\cdots , \bar{A}^{i-1}\bar{B},\cdots , \bar{A}^{L-1}\bar{B}\right)
\end{equation}

Then, the discrete SSMs \eqref{discret_ssm} are reformulated into the convolutional representation of discrete SSMs Eq. \eqref{convolution_ssm}. Besides, the convolutional view of continuous SSMs is given in Appendix B.1.
\begin{align}
	x = K \ast u, \quad
	y = \bar{C}x + \bar{D}u.
	\label{convolution_ssm}
\end{align}
where $u=\{u_1, ..., u_L\}$, $x=\{x_1, ..., x_L\}$, and $y=\{y_1, ..., y_L\}$ are discrete sequences sampled from continuous functions $u(t)$, $x(t)$, and $y(t)$.

We usually assume $x_0=0$, since the ground-true value of the initial state $x_0$ is unknown. Then, the bias term $\bar{A}^{k}x_{0}$ becomes an error term between the estimated state $\hat{x}_{k}=K \ast u$ and ground-truth state $x_k$. Fortunately, based on Proposition \ref{thm:state_convergence}, we can theoretically ignore the influence of this error term.

\begin{proposition}[State Convergence]\label{thm:state_convergence}
	If all eigenvalues of $A$ have a negative real part, for any given initial state $\hat{x}_0$, the estimated state $\hat{x}_{k}$ would convergent to its real values $x_k$ over time.
\end{proposition}
\begin{proof}
	All eigenvalues of $A$ have negative real parts, so we can get $0< \bar{A} < 1$. Thus, the state converges as bias $\bar{A}^{k}x_{0}$ decays to 0 over time, i.e., $\lim_{k \to \infty } \left| x_k- \hat{x}_{k}\right|=\lim_{k \to \infty }\bar{A}^{k}x_{0} =  0$.
\end{proof}

Convolutional representation allows us to infer and train SSM in parallel. Nevertheless, there are also the following challenges if naively applying the convolutional SSM, which still have high time and space complexity.

\textbf{Challenge in Computing $K$:}  Computing the state kernel $K$ in Eq. \eqref{state_kernel} requires $L$-th power operation on square matrix A. This process requires a lot of computing resources and time, as shown in Fig. \ref{fig:ablation}.

\textbf{Challenge in Tensor Convolution:} After obtaining the state kernel, the convolution calculation $K \ast u$ in Eq. \eqref{convolution_ssm} involves two tensors. In S4 and S4D, which are based on SISO SSM, $K$ and $u$ are univariate sequences and can be efficiently calculated by Fast Fourier Transform (FFT). How to efficiently calculate this tensor convolution is a problem that needs to be solved.
\section{Methodology: Efficient SSM}
\label{essm sec: Methodology}


\subsection{Uncouple SSM via Diagonalization}
The state kernel $K$ is the premise for the convolutional SSM \eqref{convolution_ssm}. We firstly rewrite the state kernel $K$ as $K=V\bar{B}$, where $V$ is the system kernel defined as follows.

\begin{equation}\label{state_Vandermonde}
	V \triangleq \left(1, \bar{A},\cdots , \bar{A}^{i-1},\cdots , \bar{A}^{L-1}\right)
\end{equation}

Calculating $V$ is extremely expensive if the system matrix $A$ is full. To solve this, we can uncouple the MIMO SSM \eqref{continuous_ssm} into the diagonal SSM \eqref{diagonal_ssm} equivalently using the Diagonalization Equivalence Lemma.

\begin{lemma}[Diagonalization Equivalence]\label{thm:diagonalization}
	If matrix $A\in\mathbb{R}^{N\times N}$ has $N$ real and distinct eigenvalues \{$\lambda_1,...,\lambda_N$\}, then there exists an invertible matrix $T$, such that system \eqref{continuous_ssm} can be transformed equivalently into system \eqref{diagonal_ssm}.
\begin{align}
	\dot{z}(t)=\Lambda z(t) + B'u(t), \quad
	y(t) = C'z(t) + D'u(t).
	\label{diagonal_ssm}
\end{align}
where $\Lambda=T^{-1}AT=diag\{\lambda_1,...,\lambda_N\}, B'=T^{-1}B, C'=CT, D'=D$
\end{lemma}

Here, we omit the proof of Lemma \ref{thm:diagonalization}, since it is similar to the proof of Theorem 11.1 in \cite{lu2022matrix}. The diagonalization simplifies the calculation of $V$. Furthermore, the following statements that enable efficient calculation are derived:
\begin{itemize}
	\item Given diagonal $\Lambda$, its discretization $\bar{\Lambda}$ is also diagonal.
	\item If using ZOH, $\bar{\Lambda}^i=diag\{e^{i\lambda_1 \Delta},...,e^{i\lambda_N \Delta}\}$.
	\item The time complexity of $V$ is reduced from $O(LN^3)$ to $O(LN)$.
	\item The space complexity of $V$ is reduced from $O(LN^2)$ to $O(LN)$.
\end{itemize}

For the sake of convenient notation, we will not distinguish between $B, C, D$ and $B', C,' D'$ in the following text because diagonalization does not change their structure.

\subsection{Fast Tensor Convolution}
Given $K$ and $u$, the convolution operation needs $O(L^2NH)$ multiplications, which are time-consuming and unable to be calculated in parallel. The convolution theorem realizes an efficient way to calculate convolution. More concretely, Fast Fourier Transform (FFT) could reduce the time complexity from $O(L^2)$ to $O(L\log L)$ of convolution operation for a univariate polynomial with length $L$ \cite{rao2010fast}. 

However, the state kernel $K \in \mathbb{R}^{L\times N\times N}$ and input sequence $u \in \mathbb{R}^{L\times H}$ are multidimensional tensors. Thus, FFT cannot be used directly. We solve this problem by disentangling the tensor convolution.

Firstly, based on the definition of $V$ in Eq. \eqref{state_Vandermonde}, we shall rewrite the kernel convolution and further reformulate it into Eq. \eqref{rewrite_x1} based on the $'\text{Associativity with Multiplication}'$ property of convolution.
\begin{align}\label{rewrite_x1}
	x=K\ast u
	= V\bar{B}\ast u
	= V\ast \bar{B}u
\end{align}
According to the definition of convolution, we expand Eq. \eqref{rewrite_x1}. For ease of understanding, we use the notation of tensor as introduced in \cite{kolda2009tensor}. For example, the system state at time step $k$ is noted as $V_{k::}=\bar{A}^{k-1}$, and its $i$-th eigenvalue (that is the diagonal element) is expressed as $V_{kii}=\bar{\lambda}^{k-1}_i$.

The state $x_{k}$ at time $k$ is expressed as follows.
\begin{equation}
	x_{k} = \sum_{m=1 }^{k} V_{m::}\cdot (\bar{B}u)_{(k-m):}
\end{equation}

When we use diagonal SSM, the system matrix $A$ is a diagonal matrix $\Lambda$. Thus, $V_{m::} \in \mathbb{R}^{N\times N}$ is a diagonal matrix.
The multiplication between the diagonal matrix $V_{m::}$ and vector $(\bar{B}u)_{(k-m):}$ is reformulated as follows.

\begin{equation}
	x_{ki} = \sum_{m=1 }^{k} V_{mii}\cdot (\bar{B}u)_{(k-m)i} \label{daig_v}
\end{equation}

Using the definition of convolution, we shall get
\begin{equation}
	x_{:i} = V_{:ii}\ast (\bar{B}u)_{:i} \label{disent_conv}
\end{equation}

With this, we successfully disentangle the original convolution into Eq. \eqref{disent_conv}, where $V_{:ii}$ and $(\bar{B}u)_{:i1}$ are both univariate sequence. We can apply FFT to calculate Eq. \eqref{disent_conv} efficiently.

\begin{equation}\label{convolution_fft}
	x_{:i}= iFFT\{FFT\{V_{:ii}\}\cdot FFT\{(\bar{B}u)_{:i}\}\}
\end{equation}
where $iFFT$ is the inverse FFT.

\begin{proposition}[Efficient Convolution with Diagonalization]\label{thm:fast convolution}
	Using diagonalization and  fast tensor convolution strategies, the convolution operation in Eq. \eqref{convolution_ssm} can be computed in reduced time complexity $O(LN\max \{H, \log L\})$ from $O(LNH\max\{L, N\})$.
\end{proposition}
\begin{proof}
	Given full matrix $A$ and $B$, we need $O(LN^2H)$ time complexity to obtain $K$. The convolution operation between $K$ and $u$ needs $O(L^2NH)$ time complexity. The overall time complexity is $O(LNH\max\{L, N\})$. 
	
	After using the proposed two strategies, we require $O(LNH)$ time complexity for the multiplication of $Bu$, and $O(LN)$ time complexity for $V$. Given $Bu$ and $V$, we need $O(LN\log{L})$ time complexity for Eq. \eqref{disent_conv}. Thus, the overall time complexity is $O(LN\max \{H, \log L\})$.
\end{proof}

\subsection{Efficient Block Diagonal SSM}
We propose a block diagonalization strategy to improve further the model efficiency, which can be regarded as a multi-head eSSM\cite{vaswani2017attention}. Fig. \ref{fig: model} (a) illustrates the structure of Multi-Head eSSM. Multi-head eSSM constructs multi-model for large-sized vectors in multiple subspaces. Specifically, the original input with a dimension of $H$ is first split into $s$ heads with dimension $H/s$. Then, each split sub-vector is modeled by an eSSM.

The multi-head SSM can be seen as a diagonal-block system in view of the dynamical system community. Each coefficient matrix is a diagonal block matrix, and each block on the diagonal is a coefficient of the subsystem (head), as expressed in Eq. \eqref{eq: diagonal cof}. This diagonal-block system is called a parallel-connected system in the dynamical system community. The internal structure of multi-head eSSM with two blocks is schematically illustrated in Fig \ref{InternalStructureeSSM} (b).

\begin{equation}\label{eq: diagonal cof}
	\begin{split}
	\Lambda=\begin{bmatrix}
		\Lambda_1 &  &  \\
		& \ddots  &  \\
		&  & \Lambda_s
	\end{bmatrix},
	B=\begin{bmatrix}
		B_1 &  &  \\
		& \ddots  &  \\
		&  & B_s
	\end{bmatrix},
	C=\begin{bmatrix}
		C_1 &  &  \\
		& \ddots  &  \\
		&  & C_s
	\end{bmatrix},
	D=\begin{bmatrix}
		D_1 &  &  \\
		& \ddots  &  \\
		&  & D_s
	\end{bmatrix}
	\end{split}
\end{equation}

In order to mix the features from different head, the output of all heads are concatenated and linear projected as the output of Multi-Head eSSM.

\begin{equation}
	\begin{split}
		\text{Multi-Head eSSM}(u) = & W\cdot Concat(eSSM_1(u_1), 
		..., eSSM_h(u_s))+b
	\end{split}
\end{equation}
where $W \in \mathbb{R}^{M\times N}$, and $b \in  \mathbb{R}^{M}$.

The Multi-Head setting increases optionality and diversity in model selection with a smaller size. The elements on the off-diagonal are all zero. Therefore, more heads result in fewer parameters with the same dimension $H$. With the same input, latent state, and output size, the multi-head eSSM would be more parameter efficient than the original SSM with MIMO, like S5. Moreover, there is an extreme case where the original multivariate sequence is split into $H$ univariate sequences, $s=H$. In this case, the Multi-Head eSSM becomes stacked SISO SSM, the base model in S4 and DSS. We experimentally find that increasing the number of heads can accelerate training convergence and reduce over-fitting.

\subsection{Efficient Bidirectional non-causal SSM}
A causal system is a system that generates the outputs by only taking the past and present variables as input. With this, we shall obtain the following proposition.

\begin{proposition}\label{prop_causal}
	Linear SSM is a causal system.
\end{proposition}
\begin{proof}
	The discrete SSM is linear shift-invariant, and $K(t)=0$ if $t<0$. According to the definition of causal system, linear SSM is thus causal systems.
\end{proof}
As Proposition \ref{prop_causal} says, the SSM is causal model with a forward causal kernel, noted as $\overrightarrow{K}$, which only considers the historical and current sequences. Sometimes, we need to consider the whole sequence under a non-causal setting. For example, we need a non-causal model to model images. We propose the bidirectional non-causal eSSM by adding a backward kernel $\overleftarrow{K}$, which is obtained by reversing $\overrightarrow{K}$. According to Eq. \eqref{rewrite_x1}, we could assume that the forward and backward processes have the same parameter $B$ but different $V$. Thus, $\overleftrightarrow{K}$ is further reformulated as \eqref{bidirc_kernel} with forward system kernel $\overrightarrow{V}$ and backward system kernel $\overleftarrow{V}$.
\begin{align}\label{bidirc_kernel}
	\overleftrightarrow{K} = (\overrightarrow{K} + \overleftarrow{K} )= (\overrightarrow{V} + \overleftarrow{V})\bar{B}
\end{align}

The advantage of this setting that we obtian a bidirectional non-causal model without introducing additional parameters.

\subsection{Parameterization and Initialization of eSSM}
\label{para_initial}
The diagonal SSM have learnable parameters $\Lambda, B, C, D$, and a time step $\Delta$ for discretization. We introduce the parameterization and initialization of these parameters, respectively.

\textbf{Parameter $\Lambda$}. According to Proposition \ref{thm:state_convergence}, we know that all elements in $\Lambda$ must have negative real parts to ensure state convergence. Thus, we restrict $\Lambda$ with an enforcing function $f_+$, expressed as $-f_+(Re(\Lambda))+Im(\Lambda)i$, where $Re(\cdot)$ and $Im(\cdot)$ denote real and imaginary parts. The enforcing function $f_+$ outputs positive real numbers and may have many forms, such as the Gaussian function. In this work, the clipping function with maximum value 1e-3 is use. Eigenvalues of the HiPPO matrix, Eq \eqref{A-normal}, are used to initialize $\Lambda$.

\begin{align}\label{A-normal}
	\mathbf{A}^{\mathrm{Normal}} &= 
	-\begin{cases}
		(n+\frac{1}{2})^{1/2}(k+\frac{1}{2})^{1/2}, & n>k\\
		\frac{1}{2}, & n=k\\
		(n+\frac{1}{2})^{1/2}(k+\frac{1}{2})^{1/2}, & n < k
	\end{cases}.
\end{align}

\textbf{Parameter $B$ and $C$}. $B$ and $C$ are the parameters of the linear projection function. We parameterize them as learnable full matrices.

\textbf{Parameter $D$}. $D$ is parameterized as a trainable diagonal matrix, initialized by a constant 1.

\textbf{Parameter $\Delta$}. We relax it to $\Delta \in \mathbb{R}^{N}$, set it as a learnable parameter, and initialize it by randomly sampling from a bounded interval.

We give more details about parameterization and initialization in Appendix \ref{eSSM-app-para_initia}. eSSM with 1 head theoretically has the same parameter amount as S5 because they all build on MIMO SSM. However, our eSSM has fewer parameters because of the following reasons. One is that the block diagonalization strategy reduces the model parameters; another is that the bidirectional mechanism in eSSM does not add additional parameters, but S5 increases the parameter by an additional $C$; the eSSM parameters $B$ and $C$ are all real numbers, and S5 uses complex numbers, which double the parameter amount. A detailed structure comparison between eSSM and S4, S5 is given in Appendix \ref{eSSM-app-comparison}.

\begin{figure}[!t]
	\centering
	\includegraphics[width=2.5in]{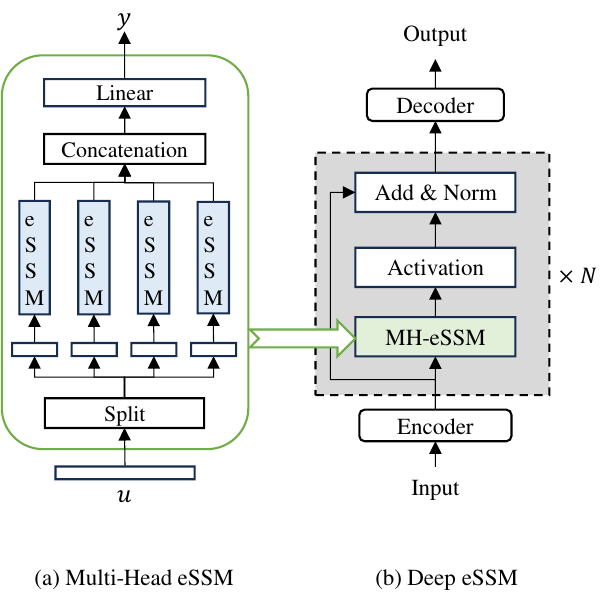}
	\caption{Architecture of deep eSSM.}\label{fig: model}
\end{figure}

\subsection{Deep eSSM}
\label{eSSM-app-deep}
The linear dynamic model, together with the activation function, can serve as a new type of neural network for addressing challenges in long-sequence modeling. We name it as Efficient State Space Model, abbreated as eSSM. This section introduces the primary settings.

The Multi-Head eSSM is stacked into multiple layers to construct the deep model for sequence modeling tasks. As shown in Fig. \ref{fig: model} (a), the activation function is applied to increase the nonlinear ability of the deep model. In addition, there is a residual connection between the Multi-Head eSSM input and the activation function output. We add (layer or batch) normalization at the end or beginning of each module.

The input sequence with length $L$ is first transformed into a sequence with a $(H, L)$ shape. Then, the proposed deep eSSM takes the encoded sequence as input and outputs a sequence with the shape of $(M, L)$. Throughout this work, we set $M=H$. Eventually, the decoder transforms the output sequence into a vector for classification. The encoder and decoder are embedding layers or linear projections.
\section{Experiments}
\label{essm sec: Experiments}

\subsection{An Illustrative Example}
In this section, we want to use a toy example to experimentally explain our proposed  Theorem \ref{thm:state_convergence} and Proposition Proposition \ref{thm:diagonalization}. Specifically, we parameterize the continuous SSM, Eq. \eqref{continuous_ssm}, with $A=[[-0.2, 1], [-1, -3]]$, $B = [[1, 0],[0, 1]]$, $C = [[1, 0],[0, 1]]$, and $D = [[0, 0],[0, 0]]$. Input is set as $u = [sin(t), cos(2*t)]^T$.

\subsubsection{Model Equivalence}
We simulate the system for 10 seconds with a time step of $\Delta=0.005$. The initial state is $[0, 0]$. Fig. \ref{essm fig: model equivalence} presents the results of running four different models, the original discrete SSM Eq. \eqref{discret_ssm}, the diagonal SSM Eq. \eqref{diagonal_ssm}, the convolution SSM directly calculated by convolution Eq. \eqref{convolution_ssm}, and the convolution SSM computed using fast tensor convolution  Eq. \eqref{convolution_fft}. The identical results of these models in Fig. \ref{essm fig: model equivalence} demonstrate their transformation equivalence.

\begin{figure}
	\centering
	\includegraphics[width=0.4\textwidth]{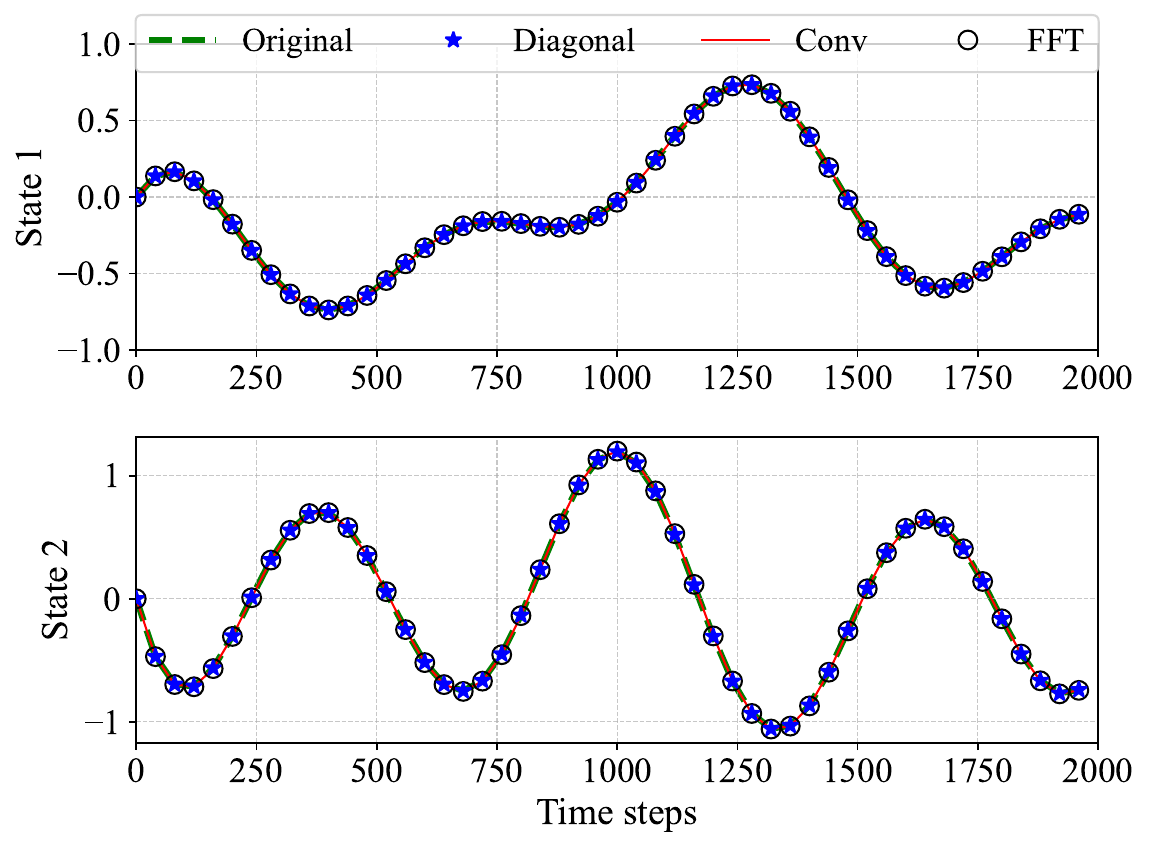}    
	\caption{Model transformation equivalence. The blue dashed line represents the original discrete SSM (Eq. \eqref{discret_ssm}), the blue asterisks represent the diagonal SSM (Eq. \eqref{diagonal_ssm}), the red solid line labeled ‘Conv’ denotes the convolution SSM directly computed using convolution (Eq. \eqref{convolution_ssm}), and the circle labeled ‘FFT’ corresponds to the convolution SSM computed via fast tensor convolution (Eq. \eqref{convolution_fft}). All models yield identical results.}
	\label{essm fig: model equivalence}
	\vspace{-5pt}
\end{figure}

\subsubsection{State Convergence}
The convolution SSM Eq. \eqref{convolution_ssm} omits the bias $\bar{A}^{k}x_{0}$ caused by unknown initial state. In the experiment, we set the initial state of the convolution SSM, noted as 'biased eSSM',  to $[0,0]^T$, while the initial state of the real system, Eq. \eqref{convolution_ssm}, is $[1,0]^T$. As shown in Fig. \ref{essm fig: state convergence}, the initial value of state 1 of the biased eSSM violates the original system.
Thanks to the convergence proposition \ref{thm:state_convergence}, the state bias will gradually decay to 0 over time. As shown in the figure, the state value of the biased eSSM gradually converges to the state value of the original system.

\begin{figure}
	\centering
	\includegraphics[width=0.4\textwidth]{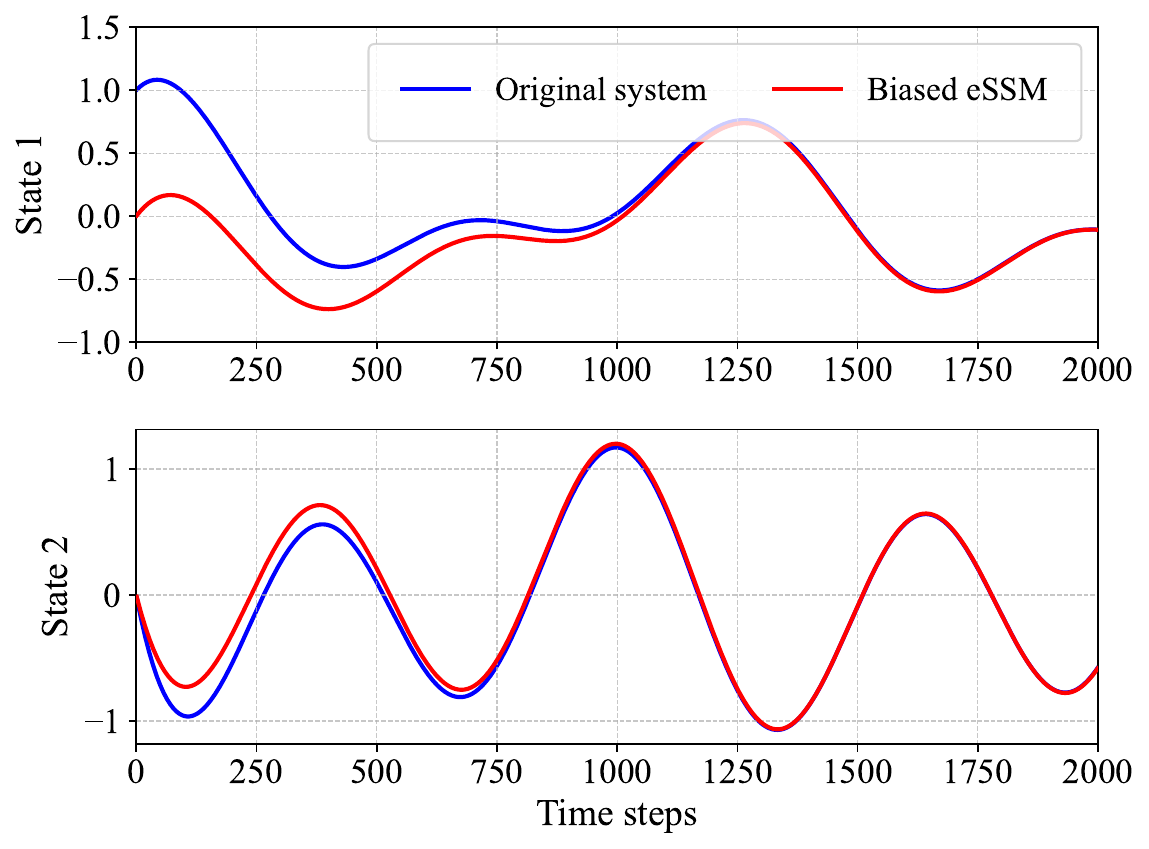}    
	\caption{State convergence. The convolutional SSM with biased initial state in red line, labeled as 'Biased eSSM' would converge to the true state of original system in blue line over time.}
	\label{essm fig: state convergence}
	\vspace{-5pt}
\end{figure}

\subsection{Experimental Settings}
\subsubsection{Datasets}
We measure the performance of the proposed eSSM on a wide range of sequence-level classification tasks, including audio, text, and images.
Specifically, Long Range Arena (LRA) \cite{tay2020long} is a challenging and standard benchmark for evaluating the models' long sequence classification performance. There are six tasks in LRA.
\begin{itemize}
	\item \textbf{LISTOPS} \cite{nangia2018listops}: It is a ten-way classification task containing a hierarchical structure and operators, and the sequence length is set as 2000.
	\item \textbf{Text}: The IMDB reviewers dataset \cite{maas2011learning} is processed and used for binary byte-level text classification with a sequence of 4096.
	\item \textbf{Retrieval}: This task measures the similarity between two sequence lengths 4000 based on the AAN dataset \cite{radev2013acl}.
	\item \textbf{Image}: $32\times32$ Images in CIFAR \cite{krizhevsky2009learning} are flattened into a sequence of length 1024 for classification tasks.
	\item \textbf{Pathfinder} \cite{linsley2018learning}: The goal is to judge if two points are connected in a $32\times32$ image. Images are flattened into a sequence of length 1024.
	\item \textbf{PathX}: It is an extreme case of Pathfinder with $128\times128$ images.
\end{itemize}

Please refer to Appendix \ref{eSSM-app-configs} for a more detailed description, including an introduction to Raw Speech Commands \cite{warden2018speech} and Pixel-level 1-D Image Classification. 

\subsubsection{Baseline Methods}
We compare the results of our proposed model with SOTA models. Those models include a diverse cross-section of long-sequence models, for example:
\begin{itemize}
	\item Transformer family: the standard vanilla Transformer \cite{vaswani2017attention}; three efficient variants, Reformer \cite{kitaev2020reformer}, Performer \cite{choromanski2020rethinking}, and Luna-256 \cite{ma2021luna}.
	\item F-net \cite{lee2021fnet}: a FFT based attention-free model.
	\item CDIL-CNN \cite{cheng2023classification}: a multi-scale model using Circular Dilated CNN.
	\item SSMs family: S4 \cite{gu2021efficiently},  S5 \cite{smith2022simplified}.
\end{itemize}

For a fair comparison, all results of baseline models are directly reused as reported in \cite{tay2020long, gu2021efficiently, cheng2023classification}.

\subsubsection{Evaluation metrics}
The performance of each model is evaluated by accuracy, which is expressed as follows
\begin{equation}
    \label{essm accuracy}
    \text{Accuracy} = \frac{1}{N}\sum_{i}^{N} 1\left( {y_i=\hat{y}_i} \right)
\end{equation}
where $N$ is the number of samples, $y_i$ is the target values, and $\hat{y}_i$ is the model predictions.

\subsubsection{Hyperparameter}
We discretize all continuous SSMs using ZOH. Key hyperparameters of all our experiments are given in Table \ref{tab: hyper_para} in Appendix \ref{eSSM-app-configs}.

\subsection{Efficiency Benchmarks}
We use the IMDB dataset with batch size 16, to compare and explore the efficiency of different models, including Transformer, LSTM, Mamba, etc. All models and experiments run 10 epochs, with the number of model parameters and average training time per epoch as evaluation metrics.

\subsubsection{Scaling Input Size}
We scale different input sizes, i.e., $H = \{64, 128, 256\}$, with a sequence length of 4096, and compare different model parameter sizes and average training time per epoch.
As shown in Fig.\ref{essm-fig: efficient comparison} (a), our eSSM has the least parameters under all input sizes. In particular, as the input size increases, the advantage of eSSM becomes more significant.
For instance, when the input size is 256, our model's parameter size is only 8.59\% of Transformer, 13.24\% of Mamba, and 40.84\% of S4, clearly demonstrating its efficiency.

Fig. \ref {essm-fig: efficient comparison} (b) presents the training time advantage of our eSSM model over other models at different input sizes. Our eSSM consistently requires less training time, particularly when the input size is 256, making it 3.94 times faster than LSTM, 2.14 times faster than S4, and 1.35 times faster than Mamba.

\subsubsection{Scaling Sequence Length}
We scale different input sequence lengths from 256 to 4096 with input size 256.
Fig.\ref{essm-fig: efficient comparison} (c) shows the training time of different models as the sequence length increases. The training time of LSTM and Transformer increases significantly with the increase of sequence length. However, the training time of SSM-based models increases slowly with the increase of sequence length, among which eSSM performs the best.

\subsubsection{Bidirectional Kernel Setting}
Fig. \ref{essm fig: kernel_comparison} compares the parameters of the three models, eSSM, S4, and DSS with and without bidirectional kernels. It can be seen that S4 and DSS will introduce additional parameters when using bidirectional kernels, while eSSM will not.
\begin{figure}
	\centering
	\includegraphics[width=0.35\textwidth]{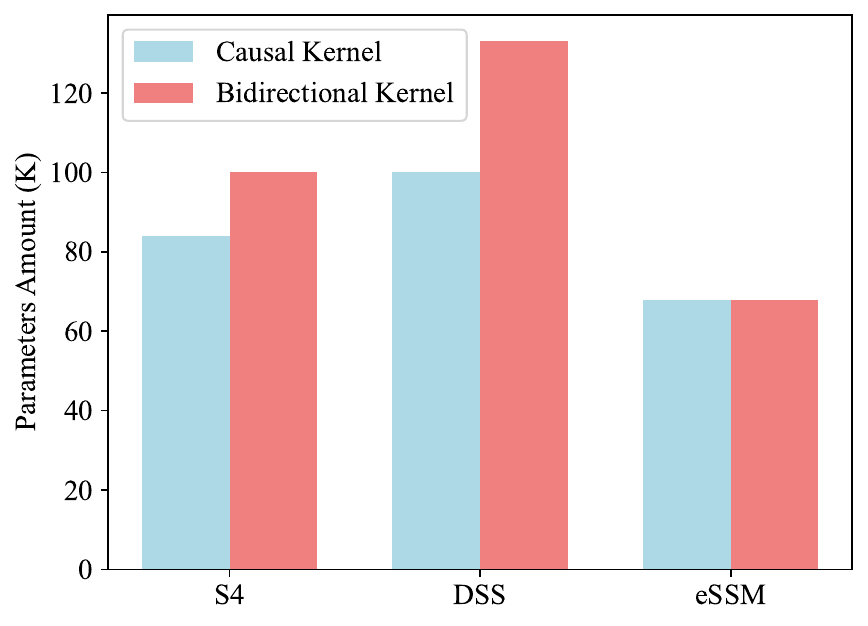}    
	\caption{Ablation study on bidirectonal kernel setting in different SMMs.}
	\label{essm fig: kernel_comparison}
	\vspace{-5pt}
\end{figure}



\subsection{Results on Long Sequence Modeling}
\begin{table*}[!b]
	\centering
	\caption{Results on the LRA benchmark tasks. The best and second best results are highlighted in bold and underlined respectively.}
	\resizebox{0.7\textwidth}{!}{%
		\begin{tabular}{cccccccc}
			\hline
			Model        & ListOps & Text  & Retrieval & Image & Pathfinder & Path-X & Avg.  \\
			(Length)       & 2,000    & 4,096  & 4,000      & 1,024  & 1,024       & 16,384  & -     \\ \hline
			Transformer \cite{vaswani2017attention}  & 36.37   & 64.27 & 57.46     & 42.44 & 71.40      & -      & 53.66 \\
			Reformer \cite{kitaev2020reformer}     	 & 37.27   & 56.10 & 53.40     & 38.07 & 68.50      & -      & 50.56 \\
			Performer \cite{choromanski2020rethinking}& 18.01   & 65.40 & 53.82     & 42.77 & 77.05      & -      & 51.18 \\ 
			Luna-256 \cite{ma2021luna}     			 & 37.25   & 64.57 & 79.29     & 47.38 & 77.72      & -      & 59.37 \\ \hline
			FNet \cite{lee2021fnet}        			 & 35.33   & 65.11 & 59.61     & 38.67 & 77.80      & -      & 54.42 \\ 
			CDIL-CNN \cite{cheng2023classification}  & 60.60   & 87.62 & 84.27     & 64.49 & 91.00      & -      & 77.59 \\ \hline
			S4 \cite{gu2021efficiently}          	 & 59.60   & 86.82 & \underline{90.90}     & \textbf{88.65} & \underline{94.2}      & \underline{96.35}  & \underline{86.09} \\
			DSS	\cite{gupta2022diagonal}			 & 60.6    & 84.8  & 87.8      & 85.7  & 84.6       & 87.8   & 81.88 \\
			S4D \cite{gu2022parameterization}        & 60.47   & 86.18 & 89.46     & \underline{88.19} & 93.06      & 91.95  & 84.89 \\
			S5 \cite{smith2022simplified}            & \underline{62.15}   & \textbf{89.31} & \textbf{91.40}     & 88.00 & \textbf{95.33}      & \textbf{98.58}  & \textbf{87.46} \\ \hline
			eSSM        							 & \textbf{62.20}  & \underline{88.25} & 90.15     & 87.25 & 93.87      &      92.76     &  85.75 \\ \hline
		\end{tabular}%
	}
 \label{essm tab:LRA_all}
\end{table*}

Table \ref{essm tab:LRA_all} presents the results of eSSM on the LRA benchmark. Our model beats the Transformer family models and the F-net in all tasks. Furthermore, SSM family modes, like S4 and S5, are state-of-the-art (SOTA) models in LRA. The overall results show that our method achieves a matchable performance to those SOTA models. Specifically, the eSSM obtains the best performance in ListOps and the second-best accuracy in Text.

\subsection{Result on Raw Speech Classification}
The raw speech classification is a 35-word audio dataset collected from real-world \cite{warden2018speech}. Table \ref{essm tab: speech 10} reports the 10-way classification results for a subset of the original dataset. eSSM is better than S4 and slightly worse than CKConv in classifying using MFCC features, which are processed based on the original sampled audio sequence. However, when modeling raw speech with a length of 16,000, eSSM significantly outperforms CKConv. Therefore, eSSM has strong capabilities in audio sequence modeling.
Results on 35-way classification are provided in Appendix \ref{eSSM-app-sup-result}.

\subsection{Zero-shot Learning on Raw Speech Classification}
In order to test the generalization performance of eSSM, following S4, we directly apply the model pre-trained on raw audio (sample at 16 KHz) to classify speech data at different sampling frequencies (8 KHz). Since the sampling frequency is reduced from 16 KHz to 8 KHz, we rescale the time step $\Delta$ in the pre-trained model to twice the original value and all other model parameters remain unchanged. As can be seen from Table \ref{essm tab: speech 10}, eSSM has good performance and achieved the second-best accuracy after S4.

\begin{table}[!ht]
	\centering
	\caption{Test accuracy on 10-way Speech Commands classification task. The best and second best results are highlighted in bold and underlined respectively.}
	\begin{tabular}{cccc}
		\hline
		Model        & MFCC  &  16kHz    & 8kHz  \\
		(Length)     & (784) & (16,000) & (8,000) \\ \hline
		Transformer  \cite{trinh2018learning, vaswani2017attention}  & 90.75 & -        & -       \\
		Performer \cite{choromanski2020rethinking}   & 80.85  & 30.77   & 30.68    \\
		ODE-RNN \cite{rubanova2019latent}     & 65.9  & -        & -       \\
		NRDE  \cite{kidger2020neural}        & 89.8  & 16.49    & 15.12   \\
		ExpRNN \cite{lezcano2019cheap}      & 82.13 & 11.6     & 10.8    \\
		LipschitzRNN \cite{erichson2020lipschitz}& 88.38 & -        & -       \\
		CKConv  \cite{romero2021ckconv}        & \textbf{95.3}  & 71.66    & 65.96   \\
		WaveGAN-D \cite{donahue2018adversarial}   & -     & 96.25    & -       \\ \hline
		LSSL  \cite{gu2021combining}       & 93.58 & -        & -       \\
		S4  \cite{gu2021efficiently}           & 93.96 & \textbf{98.32}    & \textbf{96.30}   \\ \hline
		eSSM         & \underline{94.46} &\underline{97.59}    & \underline{94.23}   \\ \hline
	\end{tabular}
 \label{essm tab: speech 10}
\end{table}

\subsection{Result on Pixel-level 1-D Image Classification}
More results on pixel-level 1-D image classification are given in Appendix \ref{eSSM-app-sup-result}, where the eSSM achieves comparable perfoemance as SOTA SSM-based models, and significantly better than transformer and RNN models.

\subsection{Ablations}
We evaluate how significantly the proposed strategies can improve the efficiency of the original MIMO SSM (noted as vanilla) in terms of fast tensor convolution and block diagonalization. We use IMDB and set the number of SSM layers to 1.

\subsubsection{Efficient Ablation on Diagonalization and Fast Tensor Convolution}
\begin{figure*}[htbp]
    \centering
    \includegraphics[width=0.68\textwidth]{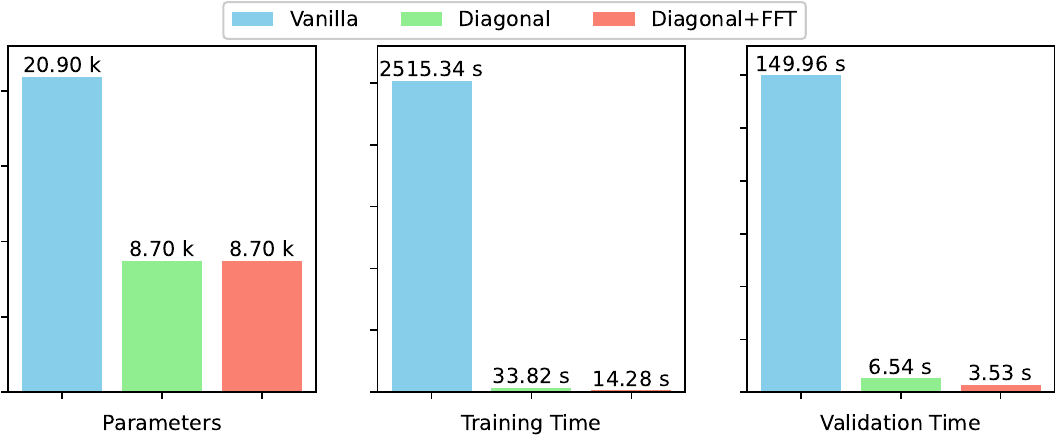}
    \caption{Results of ablation study on model efficiency. 'Vanilla' in blue is the original SSM with the full matrix $A$. 'Diagonal' in green denotes the diagonal SSM. 'Diagonal+FFT' in red denotes the diagonal SSM with FFT.}
    \label{fig:ablation}
\end{figure*}

The size of input, output, and state are all 64, that is, $H=M=N=64$. All models are trained in 10 epochs. The parameter amount, training time (per epoch), and validation (inference) time (per epoch) are used as efficiency metrics.

As shown in the bar chart in Fig. \ref{fig:ablation}, the vanilla SSM, abbreviated as 'Vanilla', is MIMO SSM with full system matrix $A$ and computed in recurrent mode. It has the worst performance in terms of time and space efficiency. After applying the diagonalization strategy, noted as 'Diagogal', model efficiency is improved significantly, with parameters reduced from 20.90k to 8.70k, training speed increased by 74.37 times, and inference speed increased by 22.93 times. Furthermore, when the fast tensor convolution is utilized, the training and inference speed can be improved further. Therefore, the proposed strategies effectively improve model efficiency with smaller model sizes and less training/inference time.


\subsubsection{Efficient Ablation on Block Diagonalization}
Table \ref{tab: ab_head} lists the results of the parameters and training time of eSSM under different block diagonalization (multi-head) configurations. The results show that increasing the number of heads can reduce the parameters of the model. In addition, under different sequence lengths, the training speed of the model can be significantly improved.

\begin{table}[]
\centering
\caption{Ablation on model block diagonalization. Model effciency is improved with the increase of head.}
\label{tab: ab_head}
\begin{tabular}{cccccc}
\hline
\multirow{2}{*}{Head} & \multirow{2}{*}{Parameter (k)} & \multicolumn{4}{c}{Training time}  \\ \cline{3-6} 
                      &                       & 512  & 1024 & 2048 & 4096 \\ \hline
4                     & 75.5                  & 42.3 & 41.6 & 47.1 & 73.8 \\
16                    & 69.3                  & 41.8 & 44.1 & 48.1 & 77.0 \\
64                    & 67.8                  & 40.3 & 41.6 & 44.4 & 62.1 \\ \hline
\end{tabular}
\end{table}

\subsection{Visualization}
To further understand the working mechanism of eSSM, we visualized the key parameters of the model, including the eigenvalues of the diagonal matrix $\Lambda$ and the discrete time intervals $\Delta$.

Fig. \ref{essm fig: lambda} shows a scatter plot of lambda in the model before and after training on the LRA database. All eigenvalues have negative real parts thanks to the enforcing function, making Proposition \ref{thm:state_convergence} holds. In most tasks, such as text and images, the real part of lambda ranges from 0 to -3. In some tasks, such as listops, a small number of smaller lambda exists in the model.

Fig. \ref{essm fig: dt} depicts a scatter plot of the $\Delta$ of trained model on the LRA dataset. To make the value of $\Delta$ significant, we use its natural logarithm as the x-axis to visualize it. In each task, the value of $\log(\Delta)$ is mainly in the range of 0 to -10, and the values of each layer are relatively similar.

\begin{figure*}
	\centering
	\includegraphics[width=0.9\textwidth]{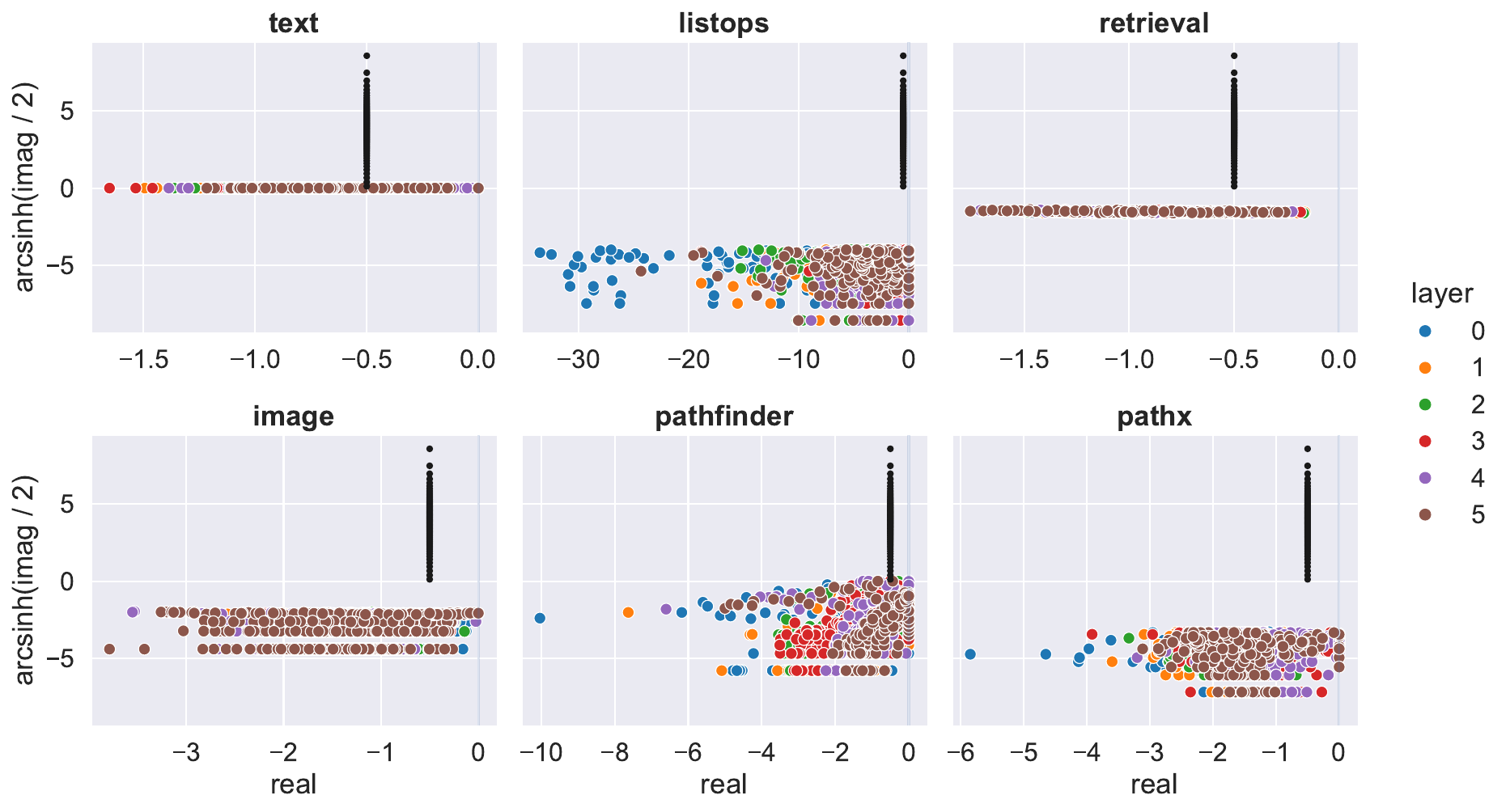}    
	\caption{Visualization of $\Lambda$ in different layers of trained eSSM on LRA. The figure uses real numbers as the x-axis and the $arcsinh$ of the imaginary value as the y-axis.Block dots denotes the Hippo initialization.}
	\label{essm fig: lambda}
	\vspace{-5pt}
\end{figure*}

\begin{figure*}
	\centering
	\includegraphics[width=0.7\textwidth]{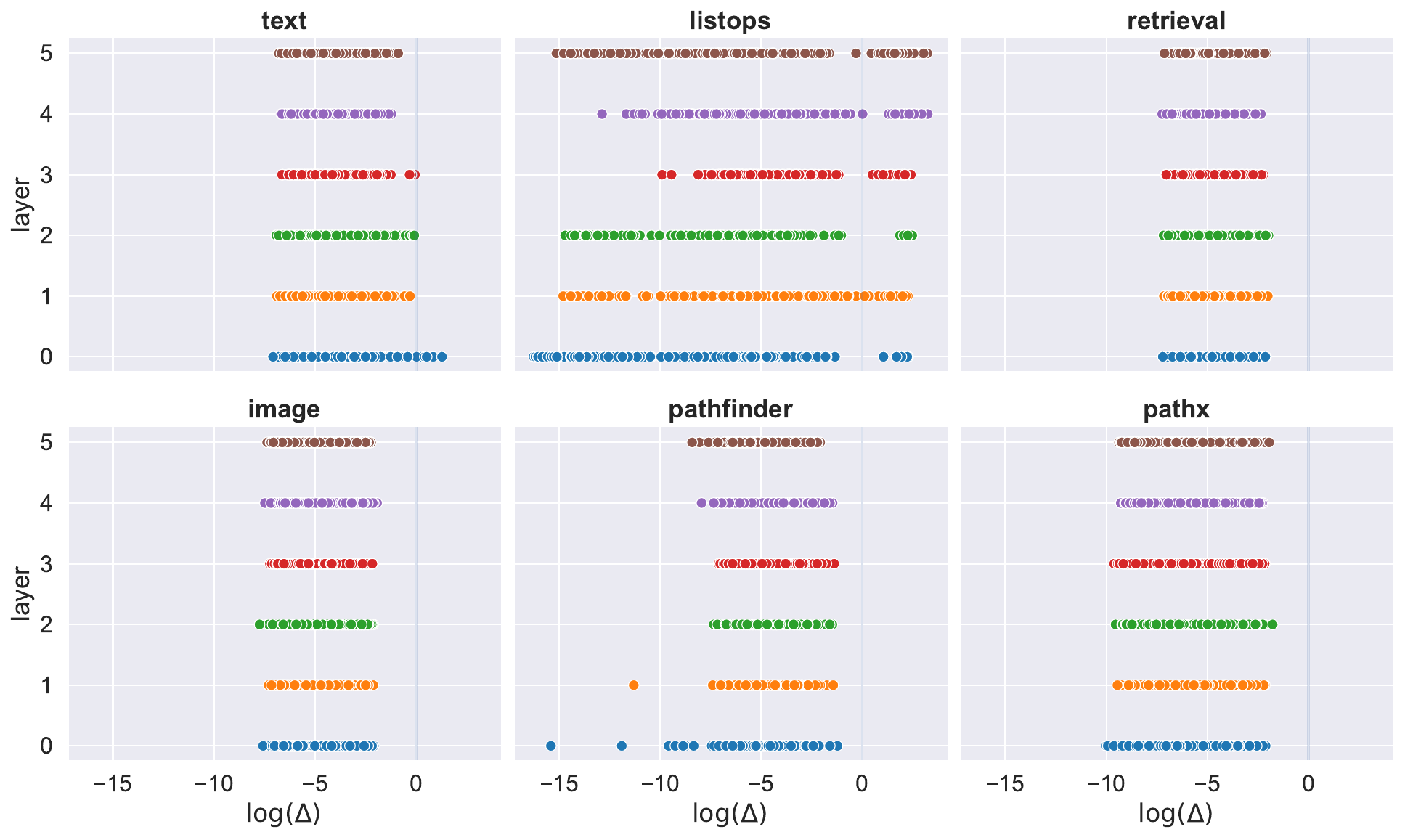}    
	\caption{Visualization of $\Delta$ of trained eSSM on LRA. We use the natural logarithm $\log(\Delta)$ as the x-axis and layers as the y-axis.}
	\label{essm fig: dt}
	\vspace{-5pt}
\end{figure*}
\section{Conclusion}
\label{sec: Conclusion}
This work proposes a new type of neural network for long-sequence modeling called efficient state space model (eSSM). The proposed eSSM builds on state space model (SSM) with multi-input and muti-output (MIMO),  a continuous dynamical model. Besides, the eSSM can be trained efficiently in convolutional representation with the help of strategies, including diagonalization and fast tensor convolution via fast Fourier transform (FFT). Furthermore, the block diagonalization and bidirectional settings improve the diversity and efficiency of the models, enabling them to be applied to various cases and tasks. We stack the eSSM as a deep model on extensive experiments. The results validate that the eSSM has matchable state-of-the-art performance with improved model efficiency. We hope this work can provide an efficient solution for modeling long sequences and provide potential directions for the theoretical design of neural networks via combining control theory.
\clearpage

\appendix



\begin{center}
    {\Large Appendix for Efficient State Space Model}
\end{center}

\medskip

\section*{Contents:}

\begin{itemize}
	\item \textbf{Appendix \ref{eSSM-app-notations}}:  Notations.
	\item \textbf{Appendix \ref{eSSM-app-model}}:  Model Details.
	\item \textbf{Appendix \ref{eSSM-app-conv_ssm}}:  Convolutional View of Continuous SSM.
	\item \textbf{Appendix \ref{eSSM-app-discretization}}:  Numerical Discretization.
	\item \textbf{Appendix \ref{eSSM-app-para_initia}}:  Parameterization and Initialization of eSSM.
	\item \textbf{Appendix \ref{eSSM-app-hippo}}:  HiPPO Initialization.
	\item \textbf{Appendix \ref{eSSM-app-comparison}}:  Comparison with Related Model.
	\item \textbf{Appendix \ref{eSSM-app-structure}}:  Structure Comparison of Different SSM.
	\item \textbf{Appendix \ref{eSSM-app-structure}}:  Parameterization and Initialization of SSM.
	\item \textbf{Appendix \ref{eSSM-app-relation}}:  Relationship between eSSM, S4, and S5.
	\item \textbf{Appendix \ref{eSSM-app-sup-result}}:  Supplementary Results.
	\item \textbf{Appendix \ref{eSSM-app-configs}}:  Experimental Configurations for Reproducibility.
	\item \textbf{Appendix \ref{eSSM-app-implementation}}:  PyTorch Implementation of eSSM Layer.
\end{itemize}

\clearpage
\section{Notations}\label{eSSM-app-notations}

\begin{table}[b]
\centering
\caption{Notations and their descriptions }
\begin{tabular}{c|l}
\toprule
\textbf{Notations}                                                  & \textbf{Descriptions}                                                                  		\\ \midrule
SSM                  						& State space model                     	 	\\
$u(t) \in \mathbb{R}^H$   					& System input sequence							\\
$x(t)\in \mathbb{R}^N$						& System state									\\
$y(t)\in \mathbb{R}^M$						& System output sequence						\\
$A \in \mathbb{R}^{N\times N}$              & System matrix in continuous SSM          	\\
$B \in \mathbb{R}^{N\times H}$              & Input matrix in continuous SSM           	\\
$C \in \mathbb{R}^{M\times N}$              & Output matrix in continuous SSM          	\\
$D \in \mathbb{R}^{M\times H}$     			& Direct transition matrix in continuous SSM   \\
$\bar{A} \in \mathbb{R}^{N\times N}$        & System matrix in discrete SSM          		\\
$\bar{B} \in \mathbb{R}^{N\times H}$        & Input matrix in discrete SSM           		\\
$\bar{C} \in \mathbb{R}^{M\times N}$        & Output matrix in discrete SSM          		\\
$\bar{D} \in \mathbb{R}^{M\times H}$     	& Direct transition matrix in discrete SSM  	\\
$\Delta\in \mathbb{R_{+}}$     				& Discrete time step in discrete SSM   		\\
$K$ 										& State kernel in convolutional SSM 			\\
$V$ 										& System kernel in convolutional SSM 			\\
$\Lambda$									& Diagonal system matrix in diagonal SSM						\\
$FFT$                                       & Fast Fourier Transform                                  \\ \bottomrule
\end{tabular}
\end{table}

\clearpage
\section{Model Details}\label{eSSM-app-model}

\subsection{Convolutional View of Continuous SSM} \label{eSSM-app-conv_ssm}
Here, we introduce the convolutional view of continuous SSM \cite{he2023unified}.
\begin{align}
	x(t)&=e^{At}x(0) + \int_{0}^{t}e^{A(t-\tau) }Bu(\tau)d\tau \notag\\
	&=e^{At}x(0)+\int_{0}^{t}e^{At}Bu(t-\tau)d\tau \label{conv_1} \\
	&=e^{At}x(0)+\int_{0}^{t}h(t)u(t-\tau)d\tau \label{conv_2} \\
	&=e^{At}x(0)+ (h\ast u)(t) \label{conv_3}
\end{align}
Using the change of variables, we reformulate Eq. \eqref{conv_1}. Then, let $h(t)=e^{At}B$, we obtain the convolutional SSM \eqref{conv_3} according to the definition of convolution.

\subsection{Numerical Discretization} \label{eSSM-app-discretization}
\subsubsection{Zero-order Hold Method}
The state transition function is an ordinary differential equation (ODE). We can obtain its analytical solution as follows.
\begin{align}
	&\dot{x}(t)=Ax(t) + Bu(t) \notag \\
	&\dot{x}(t)-Ax(t) =Bu(t) \notag\\
	&e^{-tA}\dot{x}(t)-e^{-tA}Ax(t)=e^{-tA}Bu(t) \notag\\
	&\frac{1}{dt}\left[ e^{-tA}x(t)) \right]=e^{-tA}Bu(t) \notag\\
	&\int_{0}^{t}\frac{1}{d\tau}\left[ e^{-\tau A}x(\tau) \right]d\tau=\int_{0}^{t}e^{-\tau A}Bu(\tau)d\tau \notag\\
	&e^{-tA}x(t)-x(0) =\int_{0}^{t}e^{-\tau A}Bu(\tau)d\tau \notag\\
	&x(t)=e^{tA}x(0) + e^{tA}\int_{0}^{t}e^{-\tau A}Bu(\tau)d\tau \notag\\
	&x(t)=e^{At}x(0) + \int_{0}^{t}e^{A(t-\tau) }Bu(\tau)d\tau \label{x_solution}
\end{align}

Eq. \eqref{x_solution} is the analytical solution of $x(t)$. Then, we rewrite Eq. \eqref{x_solution} with initial time $t_0$.
\begin{equation}
	x(t)=e^{A(t-t_0)}x(t_0) + \int_{t_0}^{t}e^{A(t-\tau) }Bu(\tau)d\tau \label{x_solution1}
\end{equation}

When we sample the $u(t)$ with time interval $\Delta$, $t$ becomes $k\Delta$, where $k=0,1,...$ is a positive integer. The Zero-order Hold method assumes $u(t)=u(k\Delta)$. For $t\in [k\Delta, (k+1)\Delta ]$. Thus, we have
\begin{equation}
	x((k+1)\Delta)=e^{A(\Delta)}x(k\Delta) + \int_{k\Delta}^{(k+1)\Delta}e^{A((k+1)\Delta-\tau) }d\tau Bu(k\Delta) \label{dis_solution}
\end{equation}

We abbreviate $x((k+1)\Delta)$, $x(k\Delta)$, and $u(k\Delta)$ as $x_{k+1}$, $x_k$, and $u_k$, respectively. Here, we get the discrete transition function.
\begin{equation}
	x_{k+1}=\bar{A}x_k + \bar{B}u_k \label{dis_solution1}
\end{equation}
with $\bar{A}=e^{A\Delta}$, $\bar{B}=\int_{k\Delta}^{(k+1)\Delta}e^{A((k+1)\Delta-\tau) }d\tau B$.

We can further simplify $\bar{B}$ assuming that $A$ is invertible.
\begin{align}
	\bar{B}&=\int_{k\Delta}^{(k+1)\Delta}e^{A((k+1)\Delta-\tau) }d\tau B \notag \\
	&=\int_{0}^{\Delta}e^{At}dt B \notag\\
	&=\int_{0}^{\Delta}A^{-1}\frac{de^{At}}{dt}dt B \notag\\
	&=A^{-1}(e^{At}-I)B
\end{align}

\subsubsection{Numerical Approximation}
Based on Taylor series expansion, the first derivative of $x$ can be approximated by numerical differentiation. Using the forward Euler Eq. \eqref{for_euler}, or the backward Euler Eq. \eqref{back_euler} to approximate ${x}$, we obtain the $\bar{A}$, and $\bar{B}$ as described in Eq. \eqref{discret_gbt}.
\begin{equation}
	\frac{x_{k+1}-x_k}{\Delta}\approx \dot{x}_k = Ax_t + Bu_t \label{for_euler}
\end{equation}
\begin{equation}
	\frac{x_{k+1}-x_k}{\Delta} \approx \dot{x} = Ax_{t+1} + Bu_t \label{back_euler}
\end{equation}

When we use the Generalized Bilinear Transformation method (GBT). 
\begin{equation}
	\label{discret_gbt}
	\begin{cases}
		\bar{A}= (I-\alpha A\Delta)^{-1}(I+(1-\alpha)\Delta A)\\
		\bar{B} = (I-\alpha\Delta A)^{-1}\Delta B
	\end{cases}
\end{equation}
There are three special cases for the GBT with different $\alpha$: the forward Euler method is GBT with $\alpha=0$, the Bilinear method is GBT with $\alpha=0.5$, and the backward Euler method is GBT with $\alpha=1$. Those methods approximate the differential equation based on Taylor series expansion. 


\subsection{Parameterization and Initialization of eSSM}
\label{eSSM-app-para_initia}
The diagonal SSM have learnable parameters $\Lambda, B, C, D$, and a time step $\Delta$ for discretization. We introduce the parameterization and initialization of these parameters, respectively.

\textbf{Parameter $\Lambda$}. According to Proposition 1 in Section III. A, we know that all elements in $\Lambda$ must have negative real parts to ensure state convergence. Thus, we restrict $\Lambda$ with an enforcing function $f_+$, expressed as $-f_+(Re(\Lambda))+Im(\Lambda)i$. The enforcing function $f_+$ outputs positive real numbers and may have many forms, for example, the Gaussian function, the rectified linear unit function (ReLU), and the Sigmoid function. A random or constant function can initialize $\Lambda$. Besides, it can be initialized by the eigenvalues of some specially structured matrices, such as the HiPPO matrix introduced in \cite{gu2020hippo}. We initialize $\Lambda$ via HiPPO throughout this work.

\textbf{Parameter $B$ and $C$}. $B$ and $C$ are the parameters of the linear projection function. We parameterize them as learnable full matrices. Furthermore, the initialization of $B$ is given as random numbers under HiPPO framework as introduced in section \ref{eSSM-app-hippo}. $C$ is initialized by truncated normal distribution.

\textbf{Parameter $D$}. Different parameterizations of $D$ have different meanings. If we parameterize $D$ as an untrainable zero matrix, the output of SSM is only dependent on the state. When the input and output are the same size, we can parameterize it as an identity matrix, also known as residual connection \cite{he2016deep}. In this work, $D$ is parameterized as a trainable diagonal matrix, which is initialized by a constant 1 in this work.

\textbf{Parameter $\Delta$}. $\Delta \in \mathbb{R}$ is a scalar for a given SSM. We set it as a learnable parameter and initialize it by randomly sampling from a bounded interval. This work uses [0.001, 0.1] as fault choice if not otherwise specified. We experimentally find that relaxing the size of $\Delta$ from $\mathbb{R}$ to $\mathbb{R}^N$ will improve model accuracy, which is also reported in S5 \cite{smith2022simplified}. Therefore, $\Delta \in \mathbb{R}^N$ is used across all experiments.

\subsection{HiPPO Initialization}
\label{eSSM-app-hippo}
HiPPO theory introduces a way to compress continuous signals and discrete-time series by projection onto polynomial bases \cite{gu2020hippo}. The continuous SSM, as a particular type of ordinary differential equation (ODE), also belong to this framework. Thus, the structured HiPPO matrix shall serve as a good initialization method of $\Lambda, B$. Following \cite{smith2022simplified}, we choose the HiPPO-LegS matrix for initialization, which is defined as

\begin{align}
	\mathbf{A}_{nk} &= - \begin{cases}
		(2n+1)^{1/2}(2k+1)^{1/2}, &  n>k \\
		n+1, & n = k \\
		0, &  n < k
	\end{cases}.\\
	b_{n}&=(2n+1)^{\frac{1}{2}}.
\end{align}

The naive diagonalization of $\mathbf{A}_{nk}$ to initialize $\Lambda$ would lead to numerically infeasible and unstable issues. Gu et al. \cite{gu2021efficiently} proposed that this problem is solved by equivalently transforming $\mathbf{A}_{nk}$ into a normal plus low-rank (NPLR) matrix, which is expressed as a normal matrix
\begin{align}
	\mathbf{A}^{\mathrm{Normal}}=\mathbf{V}\mathbf{\Lambda} \mathbf{V}^*
\end{align}
together with a low-rank term.
\begin{align}
	\mathbf{A}  =  \mathbf{A}^{\mathrm{Normal}} - \mathbf{P}\mathbf{Q}^{\top} =\mathbf{V}\left(\mathbf{\Lambda}- (\mathbf{V}^*\mathbf{P})(\mathbf{V}^*\mathbf{Q})^*\right)\mathbf{V}^* 
\end{align}

where unitary $\mathbf{V}\in \mathbb{C}^{N \times N}$, diagonal $\mathbf{\Lambda}\in \mathbb{C}^{N \times N}$, and low-rank factorization $\mathbf{P}, \mathbf{Q}\in \mathbb{R}^{N\times r}$.  

The HiPPO-LegS matrix can be further rewritten as
\begin{align}
	\mathbf{A}_{nk}  &= \mathbf{A}^{\mathrm{Normal}}-\mathbf{P}\mathbf{P}^{\top}
\end{align}
where
\begin{align}
	\mathbf{A}^{\mathrm{Normal}} &= 
	-\begin{cases}
		(n+\frac{1}{2})^{1/2}(k+\frac{1}{2})^{1/2}, & n>k\\
		\frac{1}{2}, & n=k\\
		(n+\frac{1}{2})^{1/2}(k+\frac{1}{2})^{1/2}, & n < k
	\end{cases}.\\
	\mathbf{P}_n &= (n+\frac{1}{2})^{\frac{1}{2}}
\end{align}

We initialize $\Lambda$ using the eigenvalue of $\mathbf{A}^{\mathrm{Normal}}$. Following S5 \cite{smith2022simplified}, the eigenvectors  of $\mathbf{A}^{\mathrm{Normal}}$ are used for $B, C$ initialization.

\clearpage
\section{Comparison with related models}
\label{eSSM-app-comparison}

\subsection{Structure Comparison of SSM}
\label{eSSM-app-structure}
According to SSM's different input and output dimensions, the current related work can be divided into two categories. One type is built on SSM with single-input single-output (SISO), including S4, DSS, and S4D; the other type is based on SSM with multi-input multi-output (MIMO), including S5 and eSSM in this work. 

As shown in Fig. \ref{InternalStructure}, SISO SSM uses univariate sequences as input and output, while MIMO SSM directly models multivariate sequences. Usually, multiple SISO SSM are used to model a multivariate sequence independently, and then a linear layer is used for feature fusion, as used in S4 and DSS. Compared with SISO SSM, MIMO SSM does not require an additional linear layer.

Table \ref{tab:compare-ssm-extend} concludes the structure of different SSM-based models. Except for S4, other methods are based on diagonal SSM. S5 is the only one that directly utilizes recursive SSM for reasoning and learning. Though other models can perform recursive reasoning, learning is based on convolutional SSM.

\begin{sidewaysfigure}[p]
	\centering
	\includegraphics[width=\columnwidth]{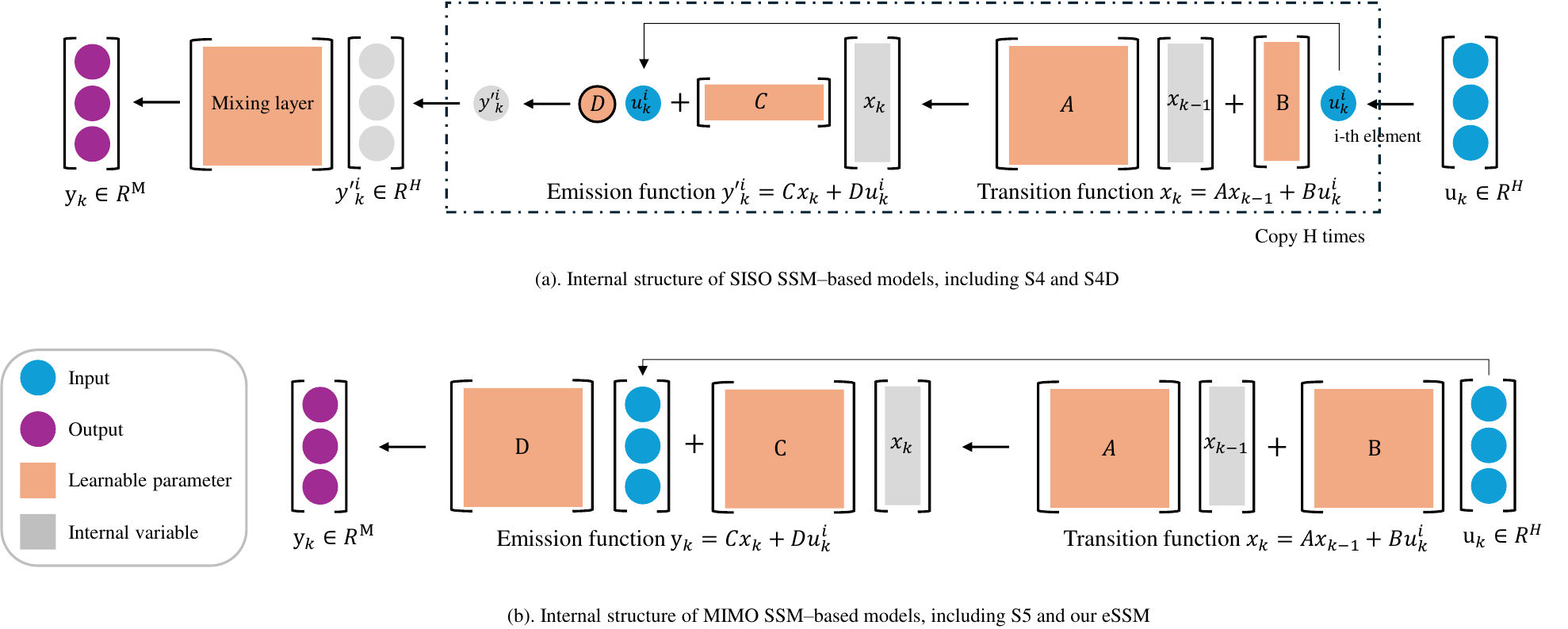} 
	\caption{The Internal Structure of SSM with SISO (a) and MIMO (b).}
	\label{InternalStructure}
\end{sidewaysfigure}

\begin{table}[!h]
	\caption{Extend Comparison of Different SSM-based Models.}
	\label{tab:compare-ssm-extend}
	\centering
	\resizebox{\columnwidth}{!}{
	\begin{tabular}{cccccccc}
		\hline
		Model & Type & Structure & Convolutional & Kernel Computation & Convolution & Recurrence & Discretiztion \\ \hline
		S4    & SISO & DPLR      &  \Checkmark   & Cauchy       	  &    FFT      &    vanilla   	   &   Bilinear    \\
		DSS   & SISO & Diagonal  &  \Checkmark   & softmax     		  &    FFT   	&    vanilla       &   ZOH 		  \\
		S4D   & SISO & Diagonal  &  \Checkmark   & Vandermonde		  &    FFT  	&    vanilla   	   &   Optional     \\ \hline
		S5    & MIMO & Diagonal  &  \XSolid      &   \XSolid          &    \XSolid 	& Scan operation   &   Bilinear     \\
		eSSM  & MIMO & Diagonal  &  \Checkmark   & Vandermonde 		  &    FFT   	& vanilla   &   ZOH     \\ \hline
	\end{tabular}
	}
\end{table}

\subsection{Relationship Between S4, S5, and eSSM}
\label{eSSM-app-relation}
S4 and S5 are the most representative works in SISO and MIMO SSM. Here, we analyze the relationship between eSSM and them. Fig \ref{fig: ComputationalFlow} presents the computational flow of those models. The following statements are summarized:
\begin{itemize}
	\item S4 is based on the SISO SSM, while S5 and eSSM are based on MIMO SSM.
	\item S4 uses a DPLR parameterization for system matrix $A$. S5 and LNDD both use diagonal SSM.
	\item All three models can make inferences in recurrent mode. However, they differ in the learning process. S4 and eSSM learn in convolutional representations, but S5 learns in recurrent representation.
	\item S4 calculates the kernel and convolution in the frequency domain, but eSSM calculates convolution in the frequency domain and the kernel in the time domain.
	\item Multi-Head LNDD and multi-copy of S4 are block-diagonal MIMO SSM, but they differ in the structure of SISO and MIMO SSM, as shown in Fig \ref{InternalStructure}.
	\item S4 with $H$ copies is the special case of Multi-Head LNDD with head number $S$ = input size $H$.
	\item S5 is equivalent to Multi-Head eSSM with head $S$=1.
	\item $A, B$, and $C$ in S4 and S5 are complex numbers, but eSSM only parameterizes $A$ as complex numbers.
	\item Bidirectional settings in eSSM do not introduce additional parameters, but S4 and S5 do.
\end{itemize}

\begin{sidewaysfigure}[p] 
	\centering
	\includegraphics[width=\columnwidth]{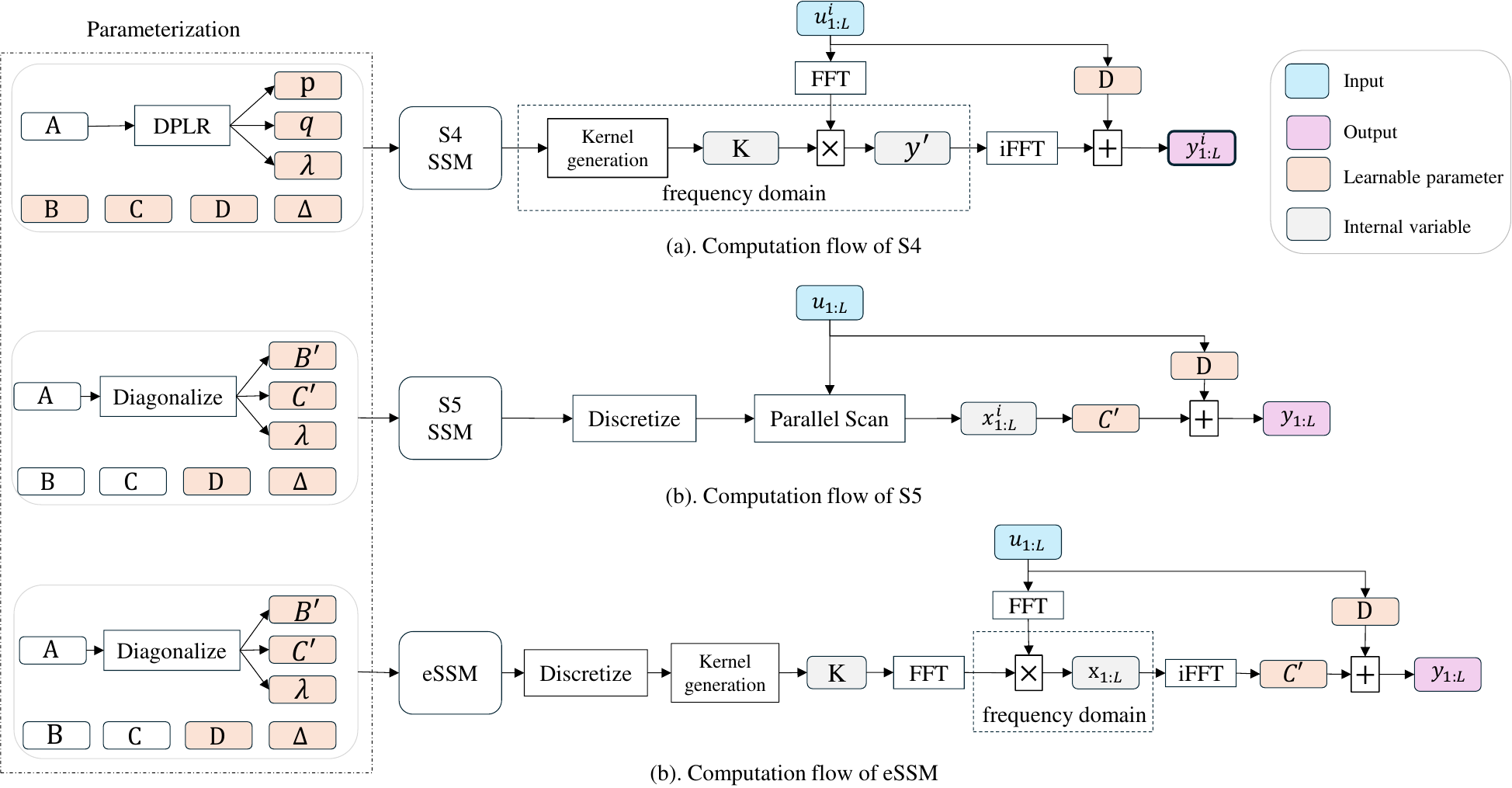} 
	\caption{The Computational Flow of S4 (a), S5 (b), and our eSSM (c).}
	\label{fig: ComputationalFlow}
\end{sidewaysfigure}

\newpage

\section{Supplementary Results}
\label{eSSM-app-sup-result}
\subsection{Extend Results on LRA}

\begin{table*}[!h]
	\centering
	\caption{Extend Results on the LRA benchmark tasks.}
	\label{tab:LRA_all}
	\begin{tabular}{cccccccc}
			\hline
			Model        & ListOps & Text  & Retrieval & Image & Pathfinder & Path-X & Avg.  \\
			length       & 2,000    & 4,096  & 4,000      & 1,024  & 1,024       & 16,384  & -     \\ \hline
			Transformer \cite{vaswani2017attention}  & 36.37   & 64.27 & 57.46     & 42.44 & 71.40      & -      & 53.66 \\
			Reformer \cite{kitaev2020reformer}     	 & 37.27   & 56.10 & 53.40     & 38.07 & 68.50      & -      & 50.56 \\
			Performer \cite{choromanski2020rethinking}& 18.01   & 65.40 & 53.82     & 42.77 & 77.05      & -      & 51.18 \\ 
			Linear Trans \cite{katharopoulos2020transformers}& 16.13   & 65.90 &  53.09     & 42.34 & 75.30      & -      & 50.46 \\

			BigBird \cite{zaheer2020big}&  36.05   & 64.02 & 59.29     &  40.83 & 74.87      & -      & 54.17 \\ 
			Luna-256 \cite{ma2021luna}     			 & 37.25   & 64.57 & 79.29     & 47.38 & 77.72      & -      & 59.37 \\ \hline
			FNet \cite{lee2021fnet}        			 & 35.33   & 65.11 & 59.61     & 38.67 & 77.80      & -      & 54.42 \\ 
			Nystr{\"o}mformer \cite{xiong2021nystromformer}    & 37.15   & 65.52 & 79.56   & 41.58 &70.94      & -      & 57.46 \\ 
			 H-Transformer-1D \cite{zhu2021h}        			 &   49.53   &  78.69 & 63.99     &  46.05 &  68.78      & -      & 61.42 \\ 
			 CCNN \cite{romero2022towards}       			 &  43.60   &  84.08 & -    &  88.90 & 91.51      & -      & 68.02 \\ 
			CDIL-CNN \cite{cheng2023classification}  & 60.60   & 87.62 & 84.27     & 64.49 & 91.00      & -      & 77.59 \\ \hline
			S4 \cite{gu2021efficiently}          	 & 59.60   & 86.82 & \underline{90.90}     & \textbf{88.65} & \underline{94.2}      & \underline{96.35}  & \underline{86.09} \\
			DSS	\cite{gupta2022diagonal}			 & 60.6    & 84.8  & 87.8      & 85.7  & 84.6       & 87.8   & 81.88 \\
			S4D \cite{gu2022parameterization}        & 60.47   & 86.18 & 89.46     & \underline{88.19} & 93.06      & 91.95  & 84.89 \\
			S5 \cite{smith2022simplified}            & \underline{62.15}   & \textbf{89.31} & \textbf{91.40}     & 88.00 & \textbf{95.33}      & \textbf{98.58}  & \textbf{87.46} \\ \hline
			eSSM        							 & \textbf{62.20}  & \underline{88.25} & 90.15     & 87.25 & 93.87      &      92.76     &  85.75 \\ \hline
	\end{tabular}%
\end{table*}

\newpage

\subsection{Extend Results on Raw Speech Classification}


\begin{table}[!h]
	\centering
	\caption{Test accuracy on 35-way Speech Commands classification task. eSSM achieves SOTA accuracy with fewest parameters.}
	\label{tab:extend speech}
	\begin{tabular}{cccc}
		\hline
		Model          & Parameters & 16kHz    & 8kHz    \\
		(Length)       &       & (16,000) & (8,000) \\ \hline
		InceptionNet \cite{nonaka2021depth}  & 481K       & 61.24    & 05.18   \\
		XResNet-50 \cite{nonaka2021depth}    & 904K       & 83.01    & 07.72   \\
		ConvNet  \cite{nonaka2021depth}      & 26.2M      & 95.51    & 07.26   \\ \hline
		S4-LegS \cite{gu2021efficiently}       & 307K       & 96.08    & 91.32   \\
		S4-FouT \cite{gu2022train}        & 307K       & 95.27    & \underline{91.59}   \\
		S4-(LegS/FouT) \cite{gu2022train} & 307K       & 95.32    & 90.72   \\
		S4D-LegS \cite{gupta2022diagonal}      & 306K       & 95.83    & 91.08   \\
		S4D-Inv \cite{gupta2022diagonal}       & 306K       & 96.18    & 91.80   \\
		S4D-Lin  \cite{gupta2022diagonal}      & 306K       & 96.25    & 91.58   \\
		Liquid-S4 \cite{hasani2022liquid}     & \underline{224K}       & \textbf{96.78}    & 90.00   \\
		S5 \cite{smith2022simplified}        & 280K       & \underline{96.52}    & \textbf{94.53}   \\ \hline
		eSSM           & \textbf{220K}       & 96.08   &  88.83   \\ \hline
	\end{tabular}
\end{table}

\newpage

\subsection{Results on Pixel-level 1-D Image classification}
\begin{table}[!h]
	\centering
	\caption{Test accuracy on Pixel-level 1-D Image classification.}
	\label{tab:extend1d-image}
	\begin{tabular}{cccc}
		\hline
		Model        & sMNIST & psMNIST & sCIFAR \\
		(Length)     & (784)  & (784)   & (1024) \\ \hline
		Transformer \cite{trinh2018learning, vaswani2017attention} & 98.9   & 97.9    & 62.2   \\
		FlexTCN \cite{romero2021flexconv}     & 99.62  & 98.63   & 80.82  \\
		CKConv  \cite{romero2021ckconv}       & 99.32  & 98.54   & 63.74  \\
		TrellisNet  \cite{bai2018trellis} & 99.20  & 98.13   & 73.42  \\
		TCN \cite{bai2018empirical}       & 99.0   & 97.2    & -      \\ \hline
		LSTM  \cite{gu2020improving,hochreiter1997long}         & 98.9   & 95.11   & 63.01  \\
		r-LSTM \cite{trinh2018learning}      & 98.4   & 95.2    & 72.2   \\
		Dilated GRU \cite{chang2017dilated} & 99.0   & 94.6    & -      \\
		Dilated RNN \cite{chang2017dilated} & 98.0   & 96.1    & -      \\
		IndRNN \cite{li2018independently}      & 99.0   & 96.0    & -      \\
		expRNN  \cite{lezcano2019cheap}     & 98.7   & 96.6    & -      \\
		UR-LSTM  \cite{gu2020improving}    & 99.28  & 96.96   & 71.00  \\
		UR-GRU  \cite{gu2020improving}     & 99.27  & 96.51   & 74.4   \\
		LMU   \cite{voelker2019legendre}       & -      & 97.15   & -      \\
		HiPPO-RNN  \cite{gu2020hippo}  & 98.9   & 98.3    & 61.1   \\
		UNIcoRNN \cite{rusch2021unicornn}    & -      & 98.4    & -      \\
		LMU-FFT  \cite{chilkuri2021parallelizing}    & -      & 98.49   & -      \\
		LipschitzRNN \cite{erichson2020lipschitz}& 99.4   & 96.3    & 64.2   \\ \hline
		LSSL \cite{gu2021combining}        & 99.53  & \textbf{98.76}   & 84.65  \\
		S4 \cite{gu2021efficiently}          & \underline{99.63}  & 98.70   & \underline{91.80}  \\
		S4D \cite{gu2022parameterization}          & -      & -       & 89.92  \\
		Liquid-S4 \cite{hasani2022liquid}   & -      & -       & \textbf{92.02}  \\
		S5  \cite{smith2022simplified}          & \textbf{99.65}  & \underline{98.67}   & 90.10  \\ \hline
		eSSM         & 99.54  & 98.45   & 88.12  \\ \hline
	\end{tabular}
\end{table}

\clearpage
\section{Experimental Configurations for Reproducibility}
\label{eSSM-app-configs}

\subsection{Hyperparameters}
Details of all experiments are described in this part. Table \ref{tab: hyper_para} lists the key hyperparameter , including model depth, learning rate, and so on.

\begin{table*}[!h]
	\centering
	\caption{Key hyperparameters.}
	\label{tab: hyper_para}
	\resizebox{\textwidth}{!}{%
		\begin{tabular}{ccccccccccccc}
			\hline
			\multirow{2}{*}{Dataset} &
			\multirow{2}{*}{Batch} &
			\multirow{2}{*}{Epoch} &
			\multirow{2}{*}{Depth} &
			\multirow{2}{*}{Head} &
			\multirow{2}{*}{H} &
			\multirow{2}{*}{N} &
			\multicolumn{1}{l}{\multirow{2}{*}{M}} &
			\multicolumn{2}{c}{LR} &
			\multirow{2}{*}{Dropout} &
			\multirow{2}{*}{Prenorm} & \\
			&     &   & && & & \multicolumn{1}{l}{} & \multicolumn{1}{l} {SSM} & Others & & \\ \hline
			ListOps    & 100  & 80 & 6 & 4 & 256 & 256 & 256& 0.01  & 0.01    & 0    &  False  \\
			Text       & 16  & 80  & 6 & 256 & 256 & 256 & 256& 0.001 & 0.004   & 0.1  &  True  \\
			Retrieval  & 32  & 20  & 6 & 64&128  & 128 & 128& 0.001 & 0.002   & 0    &  True  \\
			Image      & 50 & 200 & 6 & 64 & 256 & 512 & 256& 0.001 & 0.005    & 0.1 &  False \\
			Pathfinder & 64 & 300 & 6 & 8 & 192 & 256 & 192& 0.001 & 0.005 & 0.05  &   True  \\
			Pathx & 8 & 200 & 6 & 8 & 192 & 256 & 192 & 0.0005 & 0.001 & 0  &   True  \\ \hline
			SC10-MFCC & 16 & 80 & 4 & 32 & 128 & 128 & 128 & 0.001 & 0.006 & 0.1  &  False  \\
			SC10 & 16 & 150 & 6 & 32 & 128 & 128 & 128 & 0.001 & 0.006 &  0.1  &  True \\
			SC35 & 16 & 100 & 6 & 32 & 128 & 128 & 128 & 0.001 & 0.008 &  0.1 &  False \\  \hline
			sMNSIT & 50 & 150 & 4 & 16 & 128 & 96 & 128 & 0.002 & 0.008 & 0.1 & True  \\
			psMNIST & 50 & 200 & 4 & 8 & 128 & 128 & 128 & 0.001 & 0.004& 0.15 &  True  \\
			sCIFAR & 50 & 200 & 6 & 64 & 256 & 512 & 256& 0.001 & 0.005    & 0.1 &  True  \\  \hline
		\end{tabular}%
	}
\end{table*}

\paragraph{Activation}
MIMO SSM directly models the multivariate sequence; no additional layer is needed to mix features. Therefore, we follow S5 and use a weighted sigmoid gated unit. Specifically, the eSSM outout $\mathbf{y}_k\in \mathbb{R}^M$ is fed into the activation function expressed as $\mathbf{u}_k' = \mathrm{Gelu}(\mathbf{y}_k) \odot \sigma(\mathbf{W}*\mathrm{Gelu}(\mathbf{y}_k))$, where $\mathbf{W}\in \mathbb{R}^{M \times M}$ is a learnable dense matrix. This activation function is used as the default setting if not otherwise specified.

\paragraph{Normalization}
Either batch or layer normalization is applied before or after eSSM. Batch normalization after eSSM is used if not otherwise specified.

\paragraph{Initialization} All experiments are initialized using the same configuration introduced in \ref{eSSM-app-para_initia}.

\paragraph{Loss and Metric} Cross-entropy loss is used for all classification tasks. Binary or multi-class accuracy is used for metric evaluation.

\paragraph{Optimizer}
AdamW is used across all experiments. The learning rate applied to SSM is named $LR_{SSM}$, and the other is named $LR_{other}$. The learning rate is dynamically adjusted by $CosineAnnealingLR$ or $ReduceLROnPlateau$ in PyTorch.

\subsection{Task Specific Hyperparameters}\label{eSSM-spd-task-hp}
Here, we specify any task-specific details, hyperparameters, or architectural differences from the defaults outlined above. 

\subsubsection{Listops}
The bidirectional setting is not used. Leakyrelu activation is applied. $C$ is initialized by HiPPO.

\subsubsection{Text}
The learning rate is adjusted by $ReduceLROnPlateau$ with factor=0.5, patience=5. $LR_{other}$ is applied to SSM parameter $C$.

\subsubsection{Retrieval}
We follow the experimental configuration in S4. The model takes two documents as input and outputs two sequences. A mean pooling layer is then used to transform these two sequences into vectors, noted as $y_1$ and $y_2$. Four features are created by concatenating $y_1$ and $y_2$ as following
\begin{align}
	y = [y_1, y_2, y_1*y_2, y_1 - y_2].
\end{align}
This concatenated feature is then fed to a linear layer and gelu function for binary classification.

The learning rate is adjusted by $CosineAnnealingLR$ with warmup steps=1,000 and total training steps=50,000. $LR_{other}$ is applied to SSM parameter $C$.
\subsubsection{Image}
The learning rate is adjusted by $ReduceLROnPlateau$ with factor=0.6, patience=5. $LR_{other}$ is applied to SSM parameters $B$ and $C$. Data augmentation, including horizontal flips and random crops, is applied.

\subsubsection{Pathfinder}
The learning rate is adjusted by $CosineAnnealingLR$ with warmup steps=5,000 and total training steps=40,000. $LR_{other}$ is applied to SSM parameter $C$.

\subsubsection{Path-X}
The learning rate is adjusted by $CosineAnnealingLR$ with warmup steps=10,000 and total training steps=1,000,000. $LR_{other}$ is applied to SSM parameter $C$. $\Delta$ is initialized by uniformly sampling from [0.0001, 0.1]. 50\% training set is used before epoch=110. Validation and testing sets remain unchanged. A scale factor of 0.0625 is applied to $\Delta$.

\subsubsection{Speech Commands 10 - MFCC}
The learning rate is adjusted by $ReduceLROnPlateau$ with factor=0.2, patience=5. $LR_{other}$ is applied to SSM parameter $C$.

\subsubsection{Speech Commands 10}
The learning rate is adjusted by $ReduceLROnPlateau$ with factor=0.2, patience=10. $LR_{other}$ is applied to SSM parameter $C$.

\subsubsection{Speech Commands 35}
The learning rate is adjusted by $CosineAnnealingLR$ with total training steps=270,000.

\subsubsection{Sequential MNIST}
The learning rate is adjusted by $ReduceLROnPlateau$ with factor=0.2, patience=10. $LR_{other}$ is applied to SSM parameters $B$ and $C$.

\subsubsection{Permuted Sequential MNIST}
The learning rate is adjusted by $CosineAnnealingLR$ with warmup steps=1,000 and total training steps=81,000. $LR_{other}$ is applied to SSM parameter $C$.

\subsubsection{Sequential CIFAR}
The same hyperparameter is used as in LRA-Image.

\subsection{Dataset Details}
Here, we provide more detailed introductions to LRA, Speech Commands, and 1D image classification. This work follows the same data preprocessing process of S4 and S5. For the preprocessing details of each task, please refer to the code we provide at \href{https://github.com/leonty1/DeepeSSM}{https://github.com/leonty1/DeepeSSM}.

\subsubsection*{LRA}
\texttt{ListOps}:  
The ListOps contains mathematical operations performed on lists of single-digit integers, expressed in prefix notation \cite{nangia2018listops}. The goal is to predict each complete sequence's corresponding solution, which is also a single-digit integer. Consequently, this constitutes a ten-way balanced classification problem. For example, [MIN 2 9 [MAX 4 7 ] 0 ] has the solution 0. All sequences have a uniform length of 2000 (if not padded with zero). The dataset has a total of 10,000 samples, which are divided into 8:1:1 for training, validation, and testing. 

\texttt{Text}: 
This dataset is based on the IMDB sentiment dataset. This task aims to classify the sentiment of a given movie review (text) as either positive or negative. For example, a positive comment: 'Probably my all-time favorite movie,...'. The maximum length of each sequence is $4,096$. IMDB contains $25,000$ training examples and $25,000$ testing examples.

\texttt{Retrieval}: 
This task measures the similarity between two sequences based on the AAN dataset \cite{radev2013acl}. The maximum length of each sequence is 4,000. It is a binary classification task. There are $147,086$ training samples, $18,090$ validation samples, and $17,437$ test samples. 

\texttt{Image}: This task is based on the CIFAR-10 dataset \cite{krizhevsky2009learning}. Grayscale CIFAR-10 image has a resolution of $32\times 32$, which is flattened into a 1D sequence for a ten-way classification. All sequences have a length of 1024. It has $45,000$ training examples, $5,000$ validation examples, and $10,000$ test examples.

\texttt{Pathfinder}: This task aims to classify whether the two small circles depicted in the picture are connected with dashed lines, constituting a binary classification task  \cite{linsley2018learning}. A grayscale image has a size of $32 \times 32$, which is flattened into a sequence with length $1,024$. There are $200,000$ examples, which are split into 8:1:1 for training, validation, and testing process.

\texttt{Path-X}: A more challenging version of the \texttt{Pathfinder}.  The image's resolution was increased to $128 \times 128$, resulting in a sixteenfold increase in sequence length, from 1024 to 16,384.

\subsubsection*{Raw Speech Commands}
\texttt{Speech Commands-35}:  
This dataset records audio of 35 different words \cite{warden2018speech}. This task aims to determine which word a given audio is. It is a multi-classification problem with 35 categories. There are two audio collection frequencies, $16 KHz$ and $8 KHz$. All audio sequences have the same length, $16,000$ if sampled at $16 KHz$ or $8,000$ if sampled at $16 KHz$. It contains $24,482$ training samples, $5,246$ validation samples, and $5,247$ testing samples.

\texttt{Speech Commands-10}:  This database contains ten categories of audio, a subset of \texttt{Speech Commands-35}.

\texttt{Speech Commands-MFCC}: The original audio in \texttt{Speech Commands-10} is pre-processed into MFCC features with length of 161. 

\subsubsection*{Pixel-level 1-D Image Classification}
\texttt{Sequential MNIST} (sMNIST) :
10-way digit classification from a $28 \times 28$ grayscale image of a handwritten digit, where the input image is flattened into a $784$-length scalar sequence.

\texttt{Permuted Sequential MNIST} (psMNIST):  
This task aims to perform 10-category digit classification from a $28\times28$ grayscale image of handwritten digits. The original image is first flattened into a sequence of length 784. Next, this sequence is rearranged in a fixed order.

\texttt{Sequential CIFAR} (sCIFAR):
Color version of \texttt{image} task, where each image is an (R,G,B) triple.

\subsection{Implementation Configurations}
The experiments of accuracy are conducted with:
\begin{itemize}
    \item Operating System: Windows 10, version 22H2
    \item CPU: AMD Ryzen Threadripper 3960X 24-Core Processor @ 3.8GHz
    \item GPU: NVIDIA GeForce RTX 3090 with 24 GB of memory
    \item Software: Python 3.9.12, Cuda 11.3, PyTorch~\cite{paszke2019pytorch} 1.12.1.
\end{itemize}

The efficiency benchmark experiments are conducted with:
\begin{itemize}
    \item Operating System: Ubuntu 18.04
    \item CPU: Intel(R) Xeon(R) Platinum 8474C, 15 cores
    \item GPU: NVIDIA GeForce RTX 4090 with 24 GB of memory
    \item Software: Python 3.9.12, Cuda 11.3, PyTorch~\cite{paszke2019pytorch} 1.12.1.
\end{itemize} 
\clearpage
\section{PyTorch Implementation of eSSM Layer}
\label{eSSM-app-implementation}
\hspace*{0.2cm}
\begin{minipage}{13.8cm}
	\begin{center}
		\begin{lstlisting}[language=Python, basicstyle=\scriptsize\ttfamily, showlines=true, caption=PyTorch implementation to apply a single eSSM layer to a batch of input sequences.]
			import torch
			
			# B = batch, C = channel, S = head, H = input size, 
			# M = output size, N = state size, L = sequence length
			
			def discretize_zoh(Lambda, B, Delta):
			""" Discretize the diagonal, continuous-time linear SSM with MIMO
			Args:
			Lambda  (complex64): diagonal state matrix                      (C, S, N)
			B (complex64): input matrix            	                        (C, S, N, H)
			Delta   (float32):   discretization step sizes                  (C, S, N)
			Returns:
				discretized Lambda_bar (complex64), B_bar (complex64)           """
				Lambda_bar = Lambda * Delta
				Identity = torch.ones_like(Lambda)
				B_coef = (reciprocal(Lambda) * (torch.exp(Lambda_bar)-Identity))
				B_bar = torch.einsum('cn,cnh->cnh', B_coef, B)
				return Lambda_bar, B_bar
			
			def eSSM(Lambda_bar, B_bar, C_tilde, D, input):
			""" Discretized SSM as linear dynamic-embedded neural network.
			Args:
			input (float32): input sequence of features                   (B, H, L)
			Returns: y (float32): outputs                                 (B, M, L)    """
			
				#Split input into heads, h=H/S
				u = u.reshape(B, S, h, L)
				
				#Calculate B*u
				B_u = torch.einsum('csnh,bshl->bcsnl', B_bar, u)           
				       
				#Compute State Kernel
				length = torch.arange(Lk).cuda()
				p=torch.einsum('csn,l->csnl', Lambda_bar, length)
				state_kernel = p.exp()                                       # [channel, head, N, L]       
				state_kernel = state_kernel.real							 # real part of complex kernel
				 
				#Bidirectional kernel for non-causal state inference
				if self.bidirectional:
					#reversal backforward kernel
					state_kernel_new=F.pad(state_kernel,(0, L))+F.pad(state_kernel.flip(-1),( L, 0)) 
				else:
					state_kernel_new =state_kernel 
				
				#Efficient convolution for state  inference via FFT
				k_f = torch.fft.rfft(state_kernel_new, n=n)
				u_f = torch.fft.rfft(B_u, n=n)
				x_f = torch.einsum('bcsnl,csnl->bcsnl', u_f, k_f)
				x = torch.fft.irfft(x_f, n=n)[..., :L]
				
				#Calculate C*X
				C_x = torch.einsum('csmn, bcsnl->bcsml', C, x)
				
				#Calculate output with Du
				y = C_x + torch.einsum('csmh, bshl->bcsml', self.D, u_D)
				
				#Mix channels using linear projection
				y= dropout(y)
				y = self.channel_mixer(y)
				
				#Activation
				y=activation(y)
	
				return y
		\end{lstlisting}
	\end{center}
\end{minipage}
\clearpage

\bibliographystyle{elsarticle-num}
\bibliography{0-main}
\bio{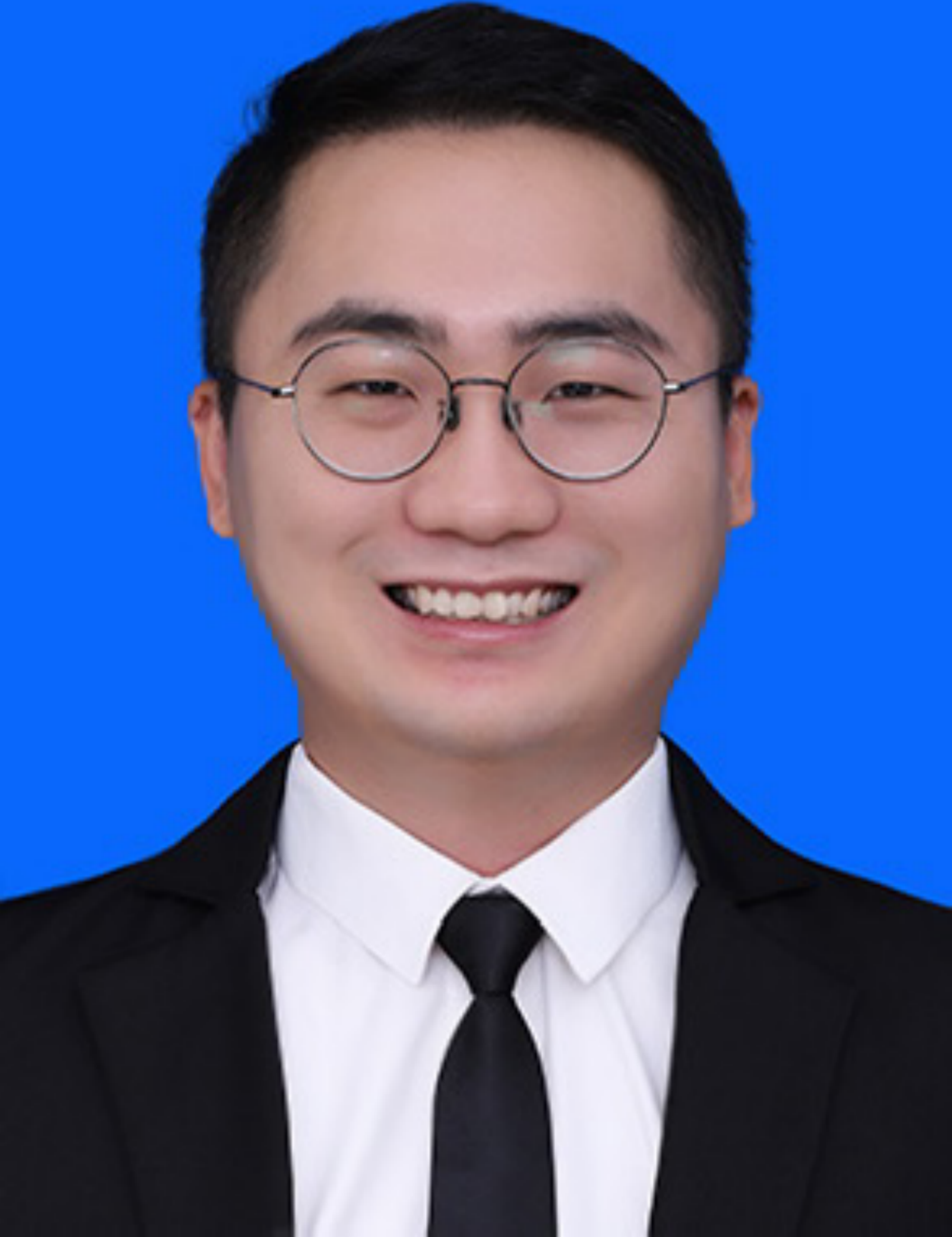}
Tongyi Liang received the B.E. degree in automotive engineering from the University of Science and Technology Beijing, Beijing, China, in 2017, the M.E. degree in automotive engineering from the Beihang University, Beijing, China, in 2020. He is currently working toward the Ph.D degree with the Department of Systems Engineering, City University of Hong Kong, Hong Kong, China. 
	
His current research interests focus on neural networks and deep learning.
\endbio

\bio{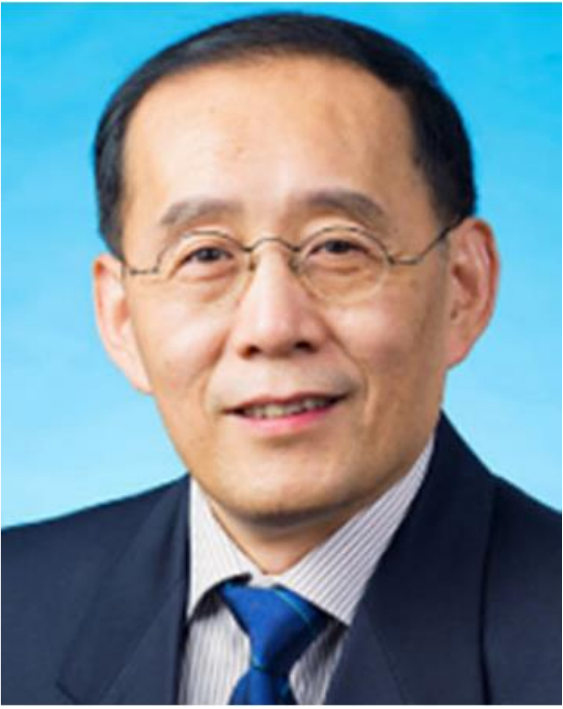}
Han-Xiong Li (Fellow, IEEE) received the B.E. degree in aerospace engineering from the National University of Defense Technology, Changsha, China, in 1982, the M.E. degree in electrical engineering from the Delft University of Technology, Delft, The Netherlands, in 1991, and the Ph.D. degree in electrical engineering from the University of Auckland, Auckland, New Zealand, in 1997.

He is the Chair Professor with the Department of Systems Engineering, City University of Hong Kong, Hong Kong. He has a broad experience in both academia and industry. He has authored two books and about 20 patents, and authored or coauthored more than 300 SCI journal papers with h-index 60 (web of science). His current research interests include process modeling and control, distributed parameter systems, and system intelligence.

Dr. Li is currently the Associate Editor for IEEE Transactions on SMC: System and was an Associate Editor for IEEE Transactions on Cybernetics (2002-–2016) and  IEEE Transactions on Industrial Electronics (2009–-2015). He was the recipient of the Distinguished Young Scholar (overseas) by the China National Science Foundation in 2004, Chang Jiang Professorship by the Ministry of Education, China in 2006, and National Professorship with China Thousand Talents Program in 2010. Since 2014, he has been rated as a highly cited scholar in China by Elsevier.
\endbio

\end{document}